\documentclass[conference]{IEEEtran}

\IEEEoverridecommandlockouts
\usepackage{cite}
\usepackage{amsmath,amssymb,amsfonts,amsthm}
\usepackage{algorithmic}
\usepackage{graphicx}
\usepackage{textcomp}
\usepackage{xcolor}
\def\BibTeX{{\rm B\kern-.05em{\sc i\kern-.025em b}\kern-.08em
    T\kern-.1667em\lower.7ex\hbox{E}\kern-.125emX}}

\usepackage{enumitem}
\usepackage{bm}
\usepackage{thmtools}

\usepackage[breaklinks=true,bookmarks=false]{hyperref}

\newtheorem{remark}{Remark}
\newtheorem{lemma}{Lemma}

\newcommand{\cov}{C}

\newcommand{\SO}{\mathrm{SO}}
\newcommand{\supp}{\mathrm{supp}}
\newcommand{\var}{\mathrm{Var}}
\newcommand{\hpsi}{\widehat{\psi}}
\newcommand{\hh}{\widehat{h}}
\newcommand{\xh}{\widehat{x}}
\newcommand{\tpsi}{\widetilde{\psi}}
\newcommand{\E}{\mathbb{E}}
\newcommand{\R}{\mathbb{R}}
\newcommand{\T}{\mathbb{T}^2}
\newcommand{\Z}{\mathbb{Z}}
\newcommand{\I}{\mathcal{I}}
\newcommand{\Lb}{\mathbf{L}}

\begin{document}

\title{Texture synthesis via projection onto multiscale, multilayer statistics
\thanks{J.H. was supported by the National Science Foundation (grant \#1845856). M.H. acknowledges funding from the National Institutes of Health (grant \#R01GM135929), the National Science Foundation (grant numbers \#1845856 and \#1912906), and the Department of Energy (grant \#DE-SC0021152).}
}

\author{\IEEEauthorblockN{Jieqian He}
\IEEEauthorblockA{\textit{Dept. of Computational Math., Science \& Engineering} \\
\textit{Michigan State University}\\
East Lansing, Michigan, USA \\
\texttt{hejieqia@msu.edu}}
\and
\IEEEauthorblockN{Matthew Hirn}
\IEEEauthorblockA{\textit{Dept. of Computational Math., Science \& Engineering} \\
\textit{Department of Mathematics} \\
\textit{Center for Quantum Computing, Science \& Engineering} \\
\textit{Michigan State University} \\
East Lansing, Michigan, USA \\
\texttt{mhirn@msu.edu}}
}

\maketitle

\begin{abstract}
   We provide a new model for texture synthesis based on a multiscale, multi-layer feature extractor. Within the model, textures are represented by a set of statistics computed from ReLU wavelet coefficients at different layers, scales and orientations. A new image is synthesized by matching the target statistics via an iterative projection algorithm. We explain the necessity of the different types of pre-defined wavelet filters used in our model and the advantages of multi-layer structures for image synthesis. We demonstrate the power of our model by generating samples of high quality textures and providing insights into deep representations for texture images. 
\end{abstract}

\section{Introduction}

The texture synthesis problem asks one to generate new, perceptually accurate, texture images given a limited (often single) realization of the texture class in question. Texture images are distinguished from other types of images in that a texture image may be modelled as a stochastic process, indicating a type of random repetition of a potentially complex pattern. As such, within the field of image generative models, the texture synthesis problem is appealing because it allows for types of statistical analysis that are not possible in general image generation. Indeed, while recent works have proposed to use generative adversarial networks (GANs) \cite{NIPS2014_5423} to perform texture synthesis and related tasks \cite{pmlr-v70-bergmann17a, Jetchev2016TextureSW, Xian2018TextureGANCD}, classically texture synthesis models fall into two categories \cite{AKL201812}:
(i) non-parametric patch rearrangement methods that extract microscopic patterns from the reference image and randomly arrange these patterns in a new image; and (ii) parametric statistic-matching models that extract a set of empirical statistics from the reference texture, and generate a new image by selecting a random image with a similar statistical profile. 

This paper addresses the second type of model based on statistical matching. Such models have two challenges: (i) What is the set of statistics needed to characterize a large class of textures? and (ii) Given the statistical profile of a reference image, how does one generate a random image with the same statistical profile? Gatys and collaborators \cite{vggsyn} had great success in addressing these two challenges by extracting the covariance statistics of the filter responses at various layers of a pre-trained convolutional neural network (CNN), the VGG-19 network \cite{vgg}, and then generating a new image with matching statistics via back-propagation and stochastic gradient descent. This work in turn inspired several subsequent methods, including \cite{Gatys2015ANA, Ulyanov2016TextureNF, Johnson2016PerceptualLF, Bortoli2019MacrocanonicalMF}.

Despite the success of \cite{vggsyn}, though, the model is not perfect and many open questions remain for statistics-based texture synthesis models. Indeed, in a follow up paper \cite{Ustyuzhaninov:whatDoesItTakeTexture2017}, it is observed that high quality texture images can be synthesized using only a one-layer network with random, multiscale filters and rectified linear unit (ReLU) nonlinearity. The combination of the two papers \cite{vggsyn, Ustyuzhaninov:whatDoesItTakeTexture2017} raises questions with respect to the trade-off between network depth and the sizes of the receptive fields of the filters in the network. Additionally, putting the use of \textit{random} filters aside, the use of a single layer of \textit{multiscale} filters parallels classical work in the field that uses the statistics of multiscale wavelet coefficients to synthesize textures \cite{heeger:pyramidTexture1995, zhu:FRAME1998, Portilla, briand:Heeger-Bergen2014}.
Among these methods, the algorithm of Portilla and Simoncelli \cite{Portilla} is particularly notable for its use of statistics based on the modulus and phase of complex wavelet valued coefficients, in addition to its impressive performance which is often still bench-marked against today. 

In this paper we propose a multiscale, multilayer, nonlinear feature extractor for images based upon real-valued wavelet transforms, which in turn yields a set of statistics for use in texture synthesis. In addition to drawing inspiration from Portilla and Simoncelli \cite{Portilla} and Gatys \textit{et al.} \cite{vggsyn}, the model presented in this paper also draws upon ideas from the wavelet scattering transform \cite{scattering}, which itself has show good results for the synthesis of gray-scale textures \cite{microcanonical}.
Nevertheless, our algorithm has several novel aspects that we use to investigate the texture synthesis problem, and which provide insight into image feature extraction via convolutional networks. More specifically:
\begin{itemize}[topsep=0pt, itemsep=0pt]
    \item We provide an analysis of the types of filters required to obtain good synthesis results when combined with the ReLU nonlinearity.
    \item We investigate the trade-off between network depth and the maximum scale of the wavelet filters.
    \item We propose a CNN architecture that uses the ReLU nonlinearity and is provably invertible at each layer, which in turn allows us to adapt the projection synthesis algorithm of \cite{Portilla} to our setting. 
    \item We demonstrate our theoretical findings numerically through example synthesized images, and also compare our results to \cite{Portilla} and \cite{vggsyn}.
\end{itemize}

In Section \ref{sec: model} we present our statistical model in detail, while Section \ref{sec: synthesis algorithm} describes our synthesis algorithm. Section \ref{sec: numerical results} provides detailed numerical results, and Section \ref{sec: conclusion} contains a short conclusion. 

\section{Model}
\label{sec: model}

Set $\T := [-T, T]^2$ and $\R_+ := [0, \infty)$, and let $x : \T \rightarrow \R_+$ be a texture image, which we shall assume is in $\Lb^2 (\T)$. A statistics-based matching algorithm for texture synthesis specifies a family of (nonlinear) functions $U_k : \Lb^2 (\T) \rightarrow \Lb^2 (\T)$ and extracts a family of empirical statistics $Sx$ from $x$ based on
\begin{equation*}
    Sx = (S_k x)_k \, , \quad S_k x := \frac{1}{(2T)^2} \int_{\T} U_k x (u) \, du \, .
\end{equation*}
A new texture $y \in \Lb^2 (\T)$ 
is synthesized by drawing $y$ from the set of images with similar statistical profiles:
\begin{equation} \label{eqn: statistics texture model}
    y \sim \I_x := \{ z \in \Lb^2 (\T) : \| Sz - Sx \| \leq \varepsilon \} \, .
\end{equation}
If $x \sim X$, where $(X(u))_{u \in \T}$ is a stochastic process, and if $U_k X$ is stationary and ergodic, then $S_k x \rightarrow \E [U_k X]$ as $T \rightarrow \infty$. Thus, we can think of $S_k x$ as approximating the statistics $\E [U_k X]$ of the unknown process that generated $x$. 

The model \eqref{eqn: statistics texture model} is appealing because the statistical profile $Sx$ determines the texture class. It is thus paramount to determine a good set of functions $(U_k)_k$, and hence statistics $(S_k)_k$, the pursuit of which has ramifications in human and computer vision \cite{BalasVision, Freeman, OkazawaE351}. 

The method of Portilla and Simoncelli \cite{Portilla} defines the majority of their statistics by leveraging a complex valued wavelet transform and extracting statistics from the modulus and phase of the resulting wavelet coefficients. The first layer of our model also uses multiscale wavelet filters, but they are real valued and we replace the modulus and phase nonlinearities with the ReLU nonlinearity. In Sections \ref{sec: wavelet filters} and \ref{sec: first layer} we explain how the proper selection of such wavelet filters, though, when combined with ReLU, can distinguish between certain types of patterns in the same way that modulus and phase can. 

On the other hand, Gatys \textit{et al.} \cite{vggsyn} define their statistics using the Gram matrices of the filter responses at various layers in the VGG-19 network. The receptive field of the filters of the VGG-19 network are small, only $3 \times 3$ pixels, but the depth and pooling of the VGG-19 network allows such statistics to still capture complex multiscale patterns in texture images. Akin to the VGG network, in Section \ref{sec: second layer} we expand our set of functions $U_k$ by computing a second wavelet transform and ReLU nonlinearity. Such a procedure is inspired by the wavelet scattering transform \cite{scattering}, but as we will describe differs from the scattering transform in several significant ways.

\subsection{Wavelet filters}
\label{sec: wavelet filters}

Let $\xh (\omega)$, for frequencies $\omega \in \Omega := \{\pi k / T : k \in \Z^2\} \subset \R^2$, denote the Fourier transform of $x$: $\xh (\omega) := \int_{\T} x (u) e^{-i u \cdot \omega} \, du$. A wavelet $\psi \in \Lb^2 (\T)$ is an oscillating waveform that is localized in both space and frequency and has zero average. Inspired by previous work in wavelet based image processing, as well as recent analyses of the filters of the VGG network \cite{GabrielssonTDA}, we make use of three types of wavelets. 

The first two are directional wavelet filters. We select one even directional filter and one odd directional filter:
\begin{align*}
    \psi^e (u) &:= g(u) \cos (\xi \cdot u) \, , \\
    \psi^o (u) &:= g(u) \sin (\xi \cdot u) \, ,
\end{align*}
where $g$ is an even window function and $\xi \in \R^2$ is the central frequency of the wavelets. These wavelets oscillate in the direction $\xi$ and have localized Fourier transforms around $\xi$ and $-\xi$. These wavelets are rotated to obtain waveforms oscillating in different directions:
\begin{equation*}
    \psi_{\theta}^{\beta} (u) := \psi^{\beta} (R_{\theta}^{-1} u) \, , \quad \beta \in \{e, o \} \, ,
\end{equation*}
where $R_{\theta} \in \SO(2)$ is the $2 \times 2$ rotation matrix about the angle $\theta \in [0, 2\pi)$. We use $M$ angles $\theta \in \Theta_M := \{ m \pi / M : 0 \leq m < M \}$. 

The third type of wavelet is based on the polar coordinate representation $u = (r, \varphi) \in [0, \infty) \times [0, 2\pi)$, and oscillates along the angle parameter $\varphi$:
\begin{equation*}
    \psi_{\ell}^p (u) := a_{\ell}(u) \cos (\ell \varphi) \, ,
\end{equation*}
where $\ell \in \Z$ is the frequency of oscillation along the angle $\varphi$. If $a_{\ell}(u) = \widetilde{a}_{\ell}(r)$, then the function $a_{\ell}$ determines the frequency of oscillation of the filter along the radial parameter. In this case, $|\hpsi^p (\omega)|$
has an essential support in the shape of an annulus and the filter is omnidirectional. We restrict $0 \leq \ell < L$ and select $a_{\ell} (u) = \widetilde{a}_{\ell}(r)$ to be an oscillatory function that oscillates at a frequency approximately proportional to $L - 1 - \ell$, ensuring that the overall frequency support of $\psi_{\ell}^p$ is approximately fixed. 

Directional filters such as $\psi_{\theta}^e$ and $\psi_{\theta}^o$ are common in image processing and various analyses of CNN filters, e.g., \cite{GabrielssonTDA},
have shown that commonly used CNNs learn directional filters. In Section \ref{sec: first layer} we will motivate the seemingly redundant choice of using both an even and odd directional wavelet filter. By examining the filters of the VGG-19 network, though, one also finds omnidirectional filters. In practice (see Section \ref{sec: numerical results}) we find that such filters improve the quality of synthesized textures in which the image patterns do not align with a small subset of directions. 

All wavelets are dilated at dyadic scales to obtain a multiscale family of waveforms:
\begin{align*}
    \psi_{j,\alpha}^{\beta} (u) &:= 2^{-2j} \psi_{\alpha}^{\beta} (2^{-j} u) \, , \quad j \in \Z \, ,\\
    (\alpha, \beta) &\in \{(\theta, e), (\theta, o), (\ell, p) \} \, .
\end{align*}
Numerically, we may assume that $\xh (\omega)$ is supported on frequencies $\omega$ contained in $[-\pi, \pi]^2$. By design the collection of wavelet filters have collective frequency support in a ball, which we can assume is the frequency ball of radius $\pi$. In this case we complement the wavelet filters with two additional filters: (i) a non-negative low pass filter $\phi \in \Lb^2 (\T)$ that has Fourier transform essentially supported around the origin (since wavelets have zero average); and (ii) a high pass filter $h \in \Lb^2 (\T)$ that has Fourier transform essentially supported outside of the frequency ball $\{ \omega \in \Omega : |\omega| \leq \pi \}$ (in other words, $\hh$ is supported in the ``corners'' of $[-\pi, \pi]^2$). The scales $2^j$ of the wavelet filters are restricted to $0 \leq j < J$, where $J \leq J_{\max} = O(\log_2 T)$, and we dilate $\phi$ to the scale $2^J$ via $\phi_J (u) := 2^{-2J} \phi (2^{-J} u)$.

The wavelet transform of this paper computes the convolution of $x$ with all the aforementioned filters:
\begin{align*}
    W_J &x := \{ x \ast \phi_J \, , \, x \ast h \, , \, x \ast \psi_{j, \alpha}^{\beta} : 0 \leq j < J, \\
    &(\alpha, \beta) \in \{ (\theta, e), (\theta, o), (\ell, p) \}, \enspace \theta \in \Theta_M, \enspace 0 \leq \ell < L \} \, .
\end{align*}

A group of filters $\{f_k\}_k$ is said to form a frame for signals $x$ such that $\supp (\xh) \subset [-\pi, \pi]^2$ if there exists two constants $0 < A \leq B < +\infty$ such that:
\begin{equation*}
    A \leq \sum_{k} |\widehat{f}_k (\omega)|^2 \leq B \, , \quad \forall \, \omega \in \Omega \cap [-\pi, \pi]^2 \, .
\end{equation*}
We can define the dual filters as $\widehat{\widetilde{f}_k}(\omega) := \frac{\overline{\widehat{f}_k (\omega)}}{ \sum_k |\widehat{f}_k (\omega)|^2 }$. An image $x$ can be reconstructed from its filtrations by the filters $\{ f_k \}_k$ using the dual filters $\{ \widetilde{f}_k \}_k$ and the formula:
\begin{equation} \label{eqn: frame inverse}
    x = \sum_k x \ast f_k \ast \widetilde{f}_k \, .
\end{equation}
For appropriately chosen parameters, the collection of filters used to define the wavelet transform $W_J$ forms a frame. As such, we can recover $x$ from $W_J x$ using \eqref{eqn: frame inverse}. This property will be important in Section \ref{sec: synthesis algorithm} for developing an algorithm by which to synthesize a new texture.

\subsection{First layer statistics}
\label{sec: first layer}

We extract directly from the image $x$ the mean, variance, skewness, and kurtosis of the image intensities $(x(u))_{u \in \T}$, in addition to the min/max intensities. We then consider the low pass filtering $x \ast \phi_J$. Since $x(u) \geq 0$ and $\phi_J (u) \geq 0$, the mean of $x \ast \phi_J$ is proportional to to the mean of $x$ and does not need to be computed. We do add in the variance of the values $(x \ast \phi_J (u))_{u \in \T}$ to $Sx$. We also add in the variance of the high pass coefficients $(x \ast h(u))_{u \in \T}$.
These statistics are also used by Portilla and Simoncelli, and one can find additional motivation for their usefulness in \cite{Portilla}. 

In order to simplify notation, let $\lambda = (j, \alpha, \beta) \in \Lambda$ be any admissible triplet for the wavelets $\psi_{j, \alpha}^{\beta}$ described in Section \ref{sec: wavelet filters}, and denote these wavelets by $\psi_{\lambda}$. Like the low pass coefficients and the high pass coefficients, we could compute only the variance of the values $(x \ast \psi_{\lambda}(u))_{u \in \T}$, but in doing so we would miss important correlations between patterns in $x$ at different scales, orientations, and angular frequencies, as captured by our wavelets. In fact, results in \cite{microcanonical} indicate that the using only the variance of wavelet coefficients does not result in good texture synthesis for certain types of textures. An alternative would be to compute the covariance between $x \ast \psi_{\lambda}$ and $x \ast \psi_{\lambda'}$, but the frequency localization of the wavelets means that such statistics will be nearly zero, and hence meaningless, for most pairs $(\lambda, \lambda')$. One possible solution is to apply a pointwise nonlinear function $\sigma : \R \rightarrow \R_+$ to the wavelet coefficients, effectively pushing the high frequencies of $x \ast \psi_{\lambda}$ down to the low frequencies for each $\lambda$, which in turn generates non-trivial correlations between $x \ast \psi_{\lambda}$ and $x \ast \psi_{\lambda'}$. In \cite{Portilla}, Portilla and Simoncelli decompose complex-valued wavelet coefficients into their modulus and phase (two nonlinear transforms), and compute covariance-type statistics of the wavelet modulus coefficients and of the phase coefficients. More recently, Zhang and Mallat \cite{Zhang2019MaximumEM} developed a wavelet phase harmonic nonlinear transform (also for complex wavelets) and used the resulting covariance statistics for texture synthesis of select gray-scale textures.

In this work we set $\sigma (t) := \max (0, t)$, which is the rectified linear unit (ReLU) nonlinearity. In order to obtain an invertible transform, we also multiply the wavelet coefficients by $\pm 1$, thus yielding the nonlinear wavelet transform:
\begin{equation*}
    U^1_J x := \{x \ast \phi_J \, , \, x \ast h \, , \,  \sigma (\gamma \cdot x \ast \psi_{\lambda} ) : \gamma = \pm 1, \enspace \lambda \in \Lambda \} \, .
\end{equation*}
Since $t = \sigma (t) - \sigma (-t)$, one can recover $W_J x$ from $U_J^1 x$ and hence one can recover $x$ using \eqref{eqn: frame inverse}. We compute the Gram matrix correlation statistics between all pairs of nonlinear coefficients in $U_J^1 x$, 
\begin{align}
    \cov_x^1 &(\lambda, \gamma, \lambda', \gamma') := \label{eqn: first layer covariance} \\ 
    &\frac{1}{(2T)^2} \int_{\T} \sigma (\gamma \cdot x \ast \psi_{\lambda} (u)) \sigma (\gamma' \cdot x \ast \psi_{\lambda'} (u)) \, du \, , \nonumber
\end{align}
with the exception of $\lambda = \lambda'$ and $\gamma = -\gamma'$ as $\cov_x^1 (\lambda, \gamma, \lambda, -\gamma) = 0$.

Note that $|t| = \sigma(t) + \sigma(-t)$, and hence the statistics \eqref{eqn: first layer covariance} subsume the covariance statistics between wavelet absolute value coefficients, which are similar to the wavelet modulus statistics computed in \cite{Portilla}. It was observed in \cite{Zhang2019MaximumEM} that the ReLU nonlinearity is related to phase information in complex valued wavelet coefficients. In \cite{Portilla}, Portilla and Simoncelli motivate the inclusion of phase by considering two one-dimensional signals, a Dirac function and a step function. These two signals cannot be distinguished by the wavelet modulus coefficients alone. ReLU wavlet coefficients, on the other hand, can distinguish a Dirac function from a step function for either even or odd wavelets. However, the next theorem shows they have trouble distinguishing the relative intensity of these functions unless even and odd wavelets are used together.

\begin{restatable}{theorem}{thmMorlet} 
\label{thm: necessity of morlet}
Define $y_1 (t) := \delta (t)$ and $y_2 (t) := \bm{1}_{[0, \infty)}(t)$. Let $\tpsi^e \in \Lb^2 (\R)$ be a one-dimensional even wavelet, and let $\tpsi^o \in \Lb^2 (\R)$ be a one-dimensional odd wavelet. Define the $2 \times 2$ correlation matrices $\cov_{y_k}^{\beta}$ as:
\begin{equation*}
    \cov_{y_k}^{\beta} (\gamma, \gamma') := \int_{\R} \sigma (\gamma \cdot y_k \ast \tpsi^{\beta} (t)) \sigma (\gamma' \cdot y_k \ast \tpsi^{\beta} (t)) \, dt \, ,
\end{equation*}
for $\beta \in \{e, o\}$ and $\gamma \in \{-1,+1\}$. Then $\cov_{y_1}^e \neq \cov_{-y_1}^e$ and $\cov_{y_2}^e = \cov_{-y_2}^e$, while $\cov_{y_1}^o = \cov_{-y_1}^o$ and $\cov_{y_2}^o \neq \cov_{-y_2}^o$.

\end{restatable}

The proof of Theorem \ref{thm: necessity of morlet} is given in Appendix \ref{sec: proofs of theorems}. It shows both even and odd wavelets are necessary in our model. For images $x$, the Dirac signal $y_1$ is similar to a dividing line that separates two regions of the same shade, which occurs in many types of texture images. This result shows that ReLU nonlinear wavelet coefficient correlations, when computed with an odd wavelet, cannot correctly determine the brightness of the dividing line relative to the regions it separates. Similarly, ReLU nonlinear wavelet coefficients, when computed with an even wavelet, cannot determine whether the color gradient across an edge is positive or negative. Numerical results illustrating these effects are given in Section \ref{sec: filter comparison}.

\subsection{Second layer statistics}
\label{sec: second layer}

ReLU wavelet correlation statistics can be complemented by two-layer statistics that are derived from feature maps that combine image information across scales before iterating the operator $U_J^1$. In particular, we compute
\begin{align*}
    U_J^2 &x := \Bigg\{ U_J^1 \left( \sum_{j=0}^{J-1} \sigma (\gamma \cdot x \ast \psi_{j, \alpha}^{\beta}) \right) : \\
    &(\alpha, \beta) \in \{ (\theta, e), (\theta, o), (\ell, p) \}, \enspace \theta \in \Theta_M, \enspace 0 \leq \ell < L \Bigg\} \, .
\end{align*}
We then compute the variance statistics of the low and high pass maps of $U_J^2 x$, and the correlation statistics between all pairs of the nonlinear wavelet maps contained in $U_J^2 x$, with the same exceptions as in the first layer. 

By iterating upon the map $U_J^1$, the map $U_J^2$ bears some similarity to the wavelet scattering transform \cite{scattering}. However, there are several important differences. As already discussed, we utilize the ReLU nonlinearity and a family of real valued wavelets, as opposed to complex valued wavelets and the modulus nonlinearity. Furthermore, the map $U_J^1$ defined here is invertible, unlike the scattering propagation operator. Finally, before iterating the map $U_J^1$, we sum over the scale index $j$ of the nonlinear maps $\sigma (\gamma \cdot x \ast \psi_{j, \alpha}^{\beta})$. This operation is akin to a $1 \times 1$ convolution operation in the VGG-19 CNN (and other CNNs), but unlike in the VGG network in which the filters being summed over have receptive fields with the same size, here we aggregate over nonlinear multiscale wavelet filtrations that allows our network to link together correlated patterns at multiple scales. This operation also has the effect of reducing the number of second layer maps, and hence statistics. Perhaps surprisingly, the operator $U_J^2$ is also invertible under appropriate conditions on the wavelets defined in Section \ref{sec: wavelet filters}. The following theorem is proved in Appendix \ref{sec: proofs of theorems}.

\begin{restatable}{theorem}{thmsecondlayerinverse} \label{thm: 2nd layer inverse}
If $\{ \phi_J, h, \psi_{j, \theta}^o : 0 \leq j < J, \enspace \theta \in \Theta_M \}$ forms a frame and if $\widehat{g}$ is non-negative, radial, and a decreasing function of $|\omega|$, then $x \mapsto \{ x \ast \phi_J \, , \, x \ast h \, , \, U_J^2 x \}$ is invertible. 
\end{restatable}

\section{Synthesis algorithm}
\label{sec: synthesis algorithm}

Since both operator $U_J^1$ and operator $U_J^2$ (combined with the low pass and high pass coefficients) are invertible, we can adapt the iterative projection algorithm of \cite{Portilla} in order to synthesize a new texture $x^{\ast}$ with approximately the same statistical profile as $x$. We describe our version of the projection algorithm in this section. 

Let us first collect the statistics described in Section \ref{sec: model}:
\begin{itemize}[topsep=0pt, itemsep=0pt]
    \item $S_J^0 x :=$ six pixel intensity statistics, given by the mean, variance, skewness, kurtosis, min, and max of $(x(u))_{u \in \T}$. 
    \item $S_J^1 x := \{ \var (x \ast \phi_J) \, , \, \var (x \ast h) \, , \, C_x^1 \}$, which are the statistics derived from the first layer coefficients. 
    \item $S_J^2 x := \{ S_J^2 x_{\phi_J} \, , \, S_J^2 x_h \, , \, C_x^2 \}$, which consists of the second layer low pass variances ($S_J^2 x_{\phi_J}$) and high pass variances ($S_J^2 x_h$), and the second layer correlation statistics between ReLU wavelet maps ($C_x^2$).
\end{itemize}

Now let $U_J^1 x_{\psi}$ denote the collection of nonlinear ReLU wavelet coefficient maps of $U_J^1 x$; let $U_J^2 x_{\phi_J}$ denote the collection of second layer low pass maps; let $U_J^2 x_h$ denote the collection of second layer high pass maps; and let $U_J^2 x_{\psi}$ denote the collection of second layer nonlinear ReLU wavelet coefficient maps. Given a reference image $x$, we start by computing its statistical profile $S_J x = (S_J^0 x, S_J^1 x, S_J^2 x)$. We then initialize our synthesized image with a random noise image $x_0$ where each $x_0 (u)$ is an i.i.d. sample from the uniform distribution.

Let $x^t$ denote the synthesized image after $t$ iterations. The algorithm first updates $x^t$ by directly modifying its intensities $(x^t(u))_{u \in \T}$ so that $S_J^0 x^t = S_J^0 x$. It then computes $U_J^1 x^t$ using the modified $x^t$. The low pass coefficients $x^t \ast \phi_J$ and the high pass coefficients $x^t \ast h$ are adjusted so that $\var (x^t \ast \phi_J) = \var (x \ast \phi_J)$ and $\var (x^t \ast h) = \var (x \ast h)$. All of these steps are computed in the exact same fashion as \cite{Portilla}. Finally, the nonlinear coefficient maps $U_J^1 x_{\psi}^t$ are adjusted to match the target correlation matrix so that $C_{x^t}^1 = C_x^1$; this step is performed in a way that is similar to how \cite{Portilla} updates the wavelet modulus coefficients. If the algorithm is using only first layer statistics, it then inverts $U_J^1 x^t$ to obtain $x^{t+1}$, and the process repeats itself.

On the other hand, if the algorithm is using second layer statistics, it then decomposes the updated maps $U_J^1 x_{\psi}^t$ further by computing $U_J^2 x^t$. The algorithm updates the collection of second layer maps $U_J^2 x^t$ by matching $S_J^2 x^t$ to $S_J^2 x$ in a similar fashion as the first layer.
At this point, the algorithm inverts the updated collection $\{ x^t \ast \phi_J \, , \, x^t \ast h \, , \, U_J^2 x^t \}$ using Theorem \ref{thm: 2nd layer inverse} to obtain $x^{t+1}$.

\section{Numerical Results}
\label{sec: numerical results}

We implement several texture synthesis experiments with the goals of (i) numerically verifying theoretical assertions made in previous sections; (ii) understanding the effect of hyper-parameter choices on the quality of synthesized textures; and (iii) comparing to other commonly used algorithms. Our texture images are taken from DTD database\footnote{https://www.robots.ox.ac.uk/~vgg/data/dtd/} and CG Texture database\footnote{https://www.textures.com/}.  Every image is resized to $256 \times 256$ for consistency. In our experiments, the wavelet transform is implemented in the frequency field using the fast Fourier transform. Therefore, we also periodize certain images to avoid border effect \cite{moisan:hal-00388020}. For directional wavelets, we fix the total number of rotations as $M=4$.
For omnidirectional wavelets, we fix the total number of oscillations at $L=4$. The maximum scale is $J_{\max} = 6$. More implementation details are given in Appendix \ref{sec: implemetation details}.
In the following subsections, we numerically prove the advantages of using all three types of filters (Section \ref{sec: filter comparison}); we examine the role of the maximum scale $2^J$ (Section \ref{sec: max scale comparison}); and we compare the one-layer synthesis to the two-layer synthesis, thus examining the role of network depth (Section \ref{sec: layers analysis}). Finally, in Section \ref{sec: methods comparison} we also compare our results to Gatys \textit{et al.} \cite{vggsyn} and Portilla and Simoncelli \cite{Portilla}. 

\subsection{Filter comparison}
\label{sec: filter comparison}

Theorem \ref{thm: necessity of morlet} proves both even and odd wavelets are important for texture synthesis. Figure \ref{fig:filter} validates this statement numerically. We see, for example, that synthesis with odd wavelets is prone to blurring edges and even flipping the colors of enclosed regions that should have the same color (e.g., the last row of Figure \ref{fig:filter}). When using only even wavelets, images with banded colors, for example the third row of Figure \ref{fig:filter}, are blurred. Synthesized images using both even and odd wavelets are generally a clear improvement over their single wavelet-type counterparts. 

\begin{figure}
    \centering
    \includegraphics[width = 0.95 \linewidth]{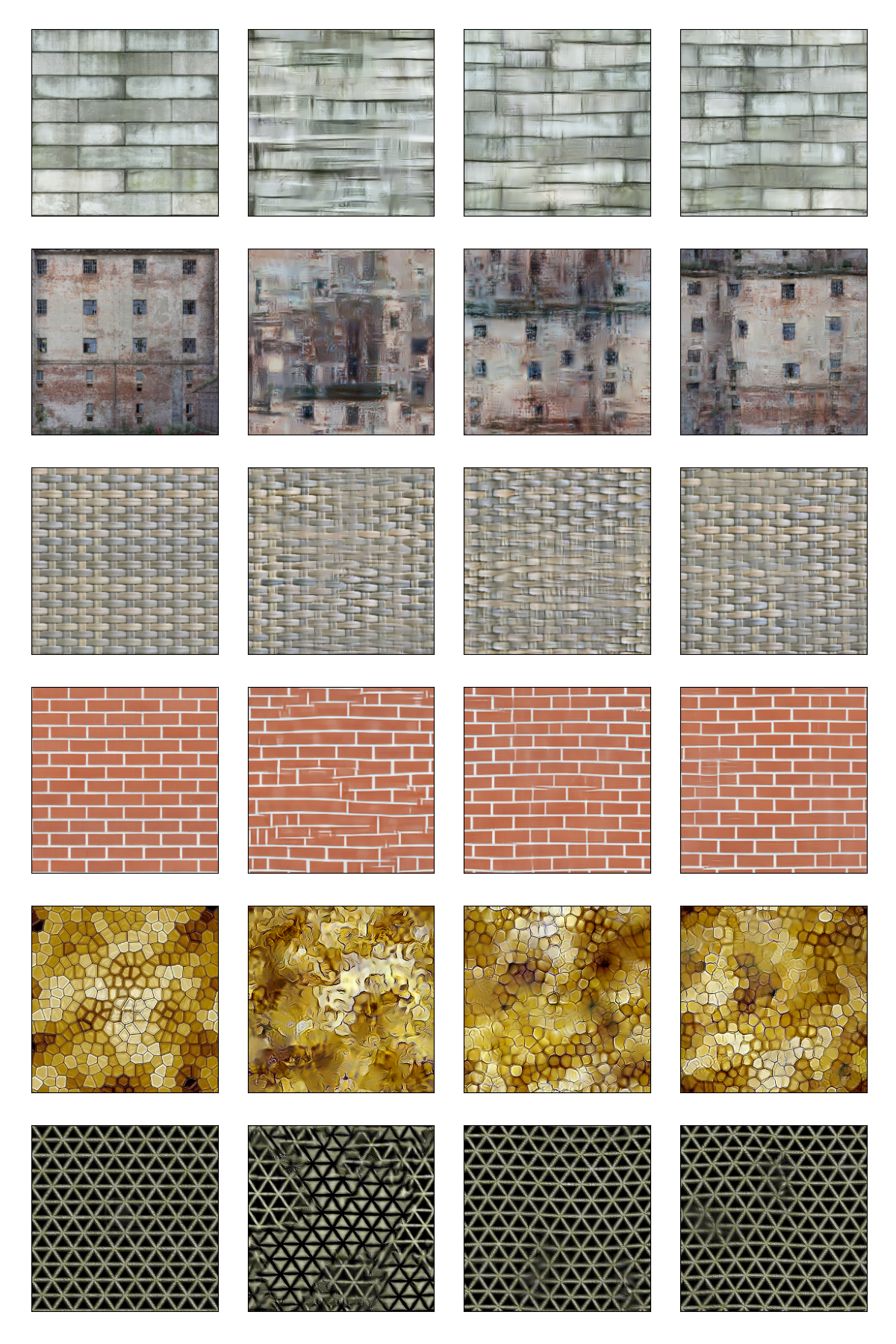}
    \caption{Synthesis results from one layer ($J=6$) with different types of wavelet filters. \textbf{Left:} Original image. \textbf{Middle left:} First layer synthesis results with only odd wavelets. \textbf{Middle right:} First layer synthesis results with only even wavelets. \textbf{Right:} First layer synthesis results with both even and odd wavelets. }
    \label{fig:filter}
\end{figure}

Figure \ref{fig:radial} shows the necessity for using omnidirectional wavelets. The advantages are clearest for rounded shapes or swirls. Such patterns have no clear direction.
For example, in the swirly texture in the last row of Figure \ref{fig:radial}, the omnidirectional wavelets do a better job of reproducing the long swirls. The round-shaped pebbles and dots (rows one and three of Figure \ref{fig:radial}) have smoother edges and a cleaner background when adding the omnidirectional wavelets.

\begin{figure}
    \centering
    \includegraphics[width = 0.77 \linewidth]{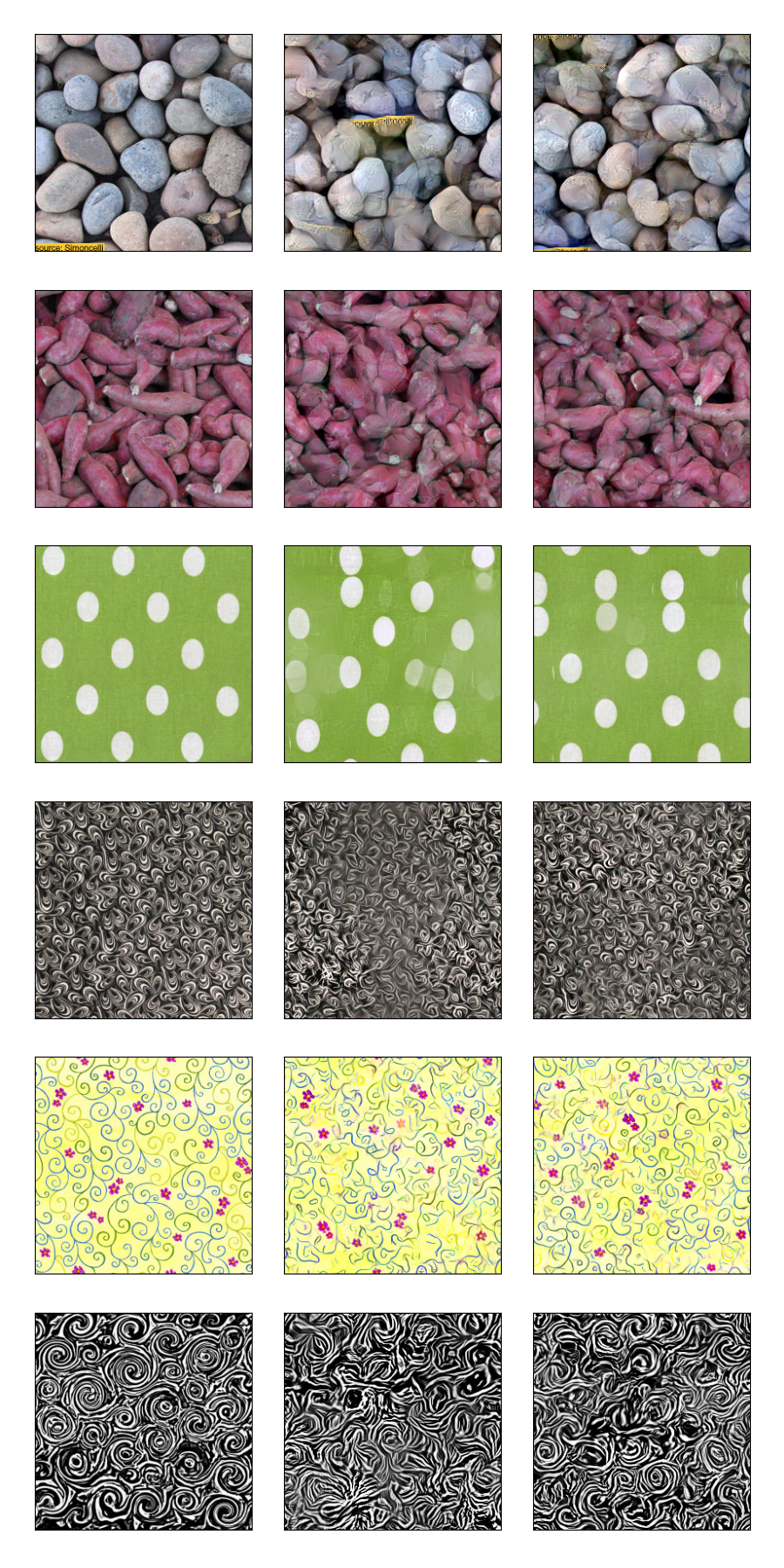}
    \caption{Synthesis results from two layers ($J=5$) with/without omnidirectional wavelets. \textbf{Left:} Original image. \textbf{Middle:} 2nd layer synthesis with even and odd wavelets. \textbf{Right:} 2nd layer synthesis with even, odd and omnidirectional wavelets.}
    \label{fig:radial}
\end{figure}

\subsection{Comparison of maximum scale}
\label{sec: max scale comparison}

Figure \ref{fig:scales} shows synthesis results with different numbers of scales from the two layer model. With $J=3$, the statistics are not able to capture macroscopic patterns.
Therefore the edges of synthesized bricks (rows one and three) and frames (second to last row) are not straight or continuous. The reproduced house (second row) and dots (third last row) are more random compared to the original image. Larger scales can also capture longer swirls (last row). However, $J=5$ and $J=6$ achieve equivalent perceptual accuracy, proving there is no need to add in larger scales than $J=5$. Indeed, the effective receptive field of the two layer synthesis with $J=5$ is equivalent to the single layer receptive field of $J_{\max} = 6$.
We also observe that using smaller scales gives more variety of the pattern arrangements, while large scale statistics have the potential to duplicate the reference image. 
\begin{figure}
    \centering
    \includegraphics[width = 0.95 \linewidth]{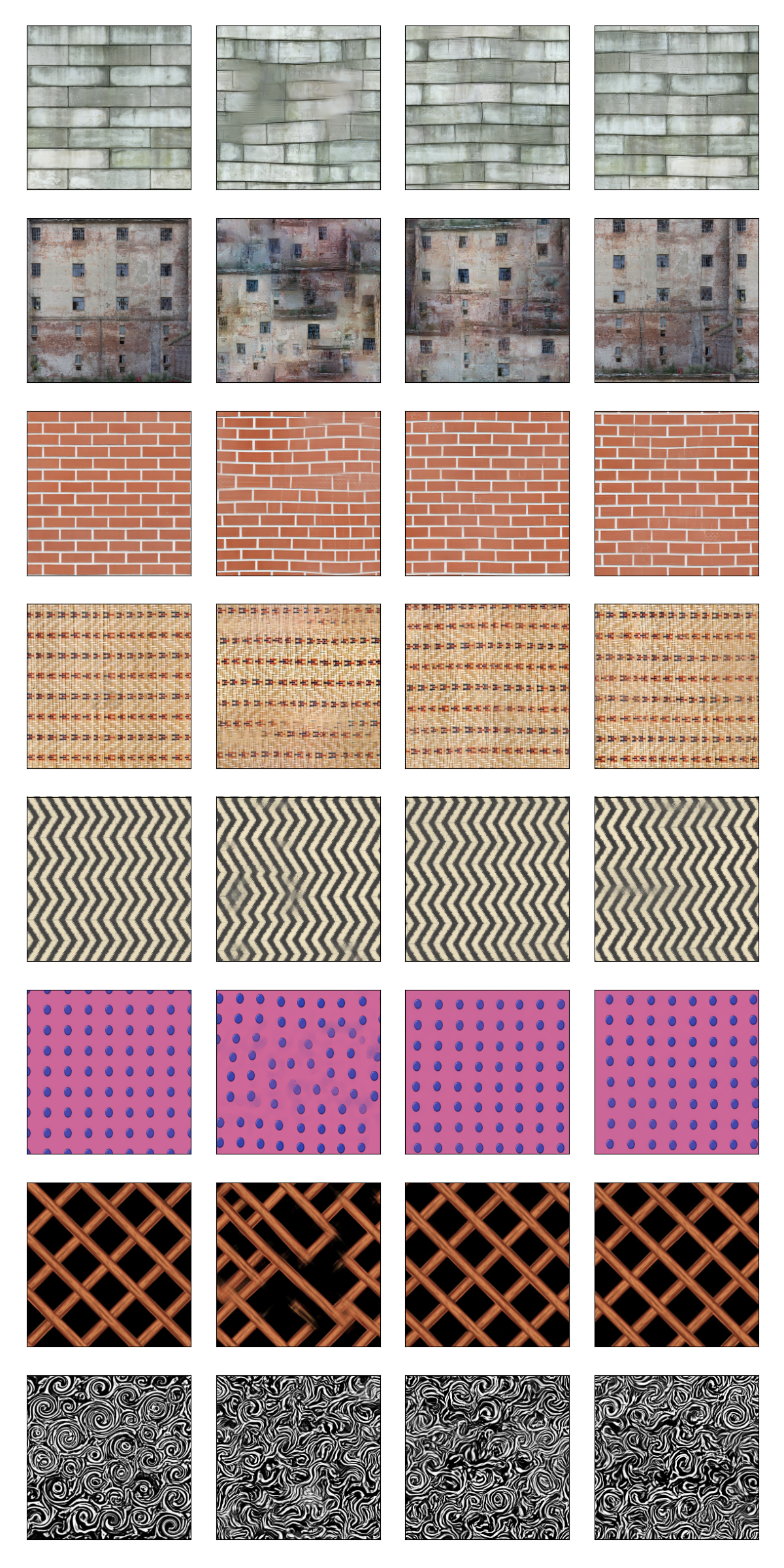}
    \caption{Synthesis results from two layers with different number of scales. \textbf{Left:} Original image. \textbf{Middle Left:} 2nd layer synthesis results with $J=4$. \textbf{Middle Right:} 2nd layer synthesis results with $J=5$. \textbf{Right:} 2nd layer synthesis results with $J=6$.}
    \label{fig:scales}
\end{figure}

\subsection{Layers analysis}
\label{sec: layers analysis}

For many texture images, the one-layer model can synthesize images of high quality. However for images with more complicated structures, multiple layers can provide a boost in visual quality. As discussed in \cite{scattering}, deeper layers decompose high frequency information that is aggregated into large frequency bins with a single wavelet transform. Figure \ref{fig:layers} shows images that achieved better quality with second layer statistics. For most images, the one-layer model captures general structures while the two-layer model refines the details, e.g., reproducing more accurate shapes, preserving long edges and swirls, fixing blurriness. 

\begin{figure}
    \centering
    \includegraphics[width = 0.75 \linewidth]{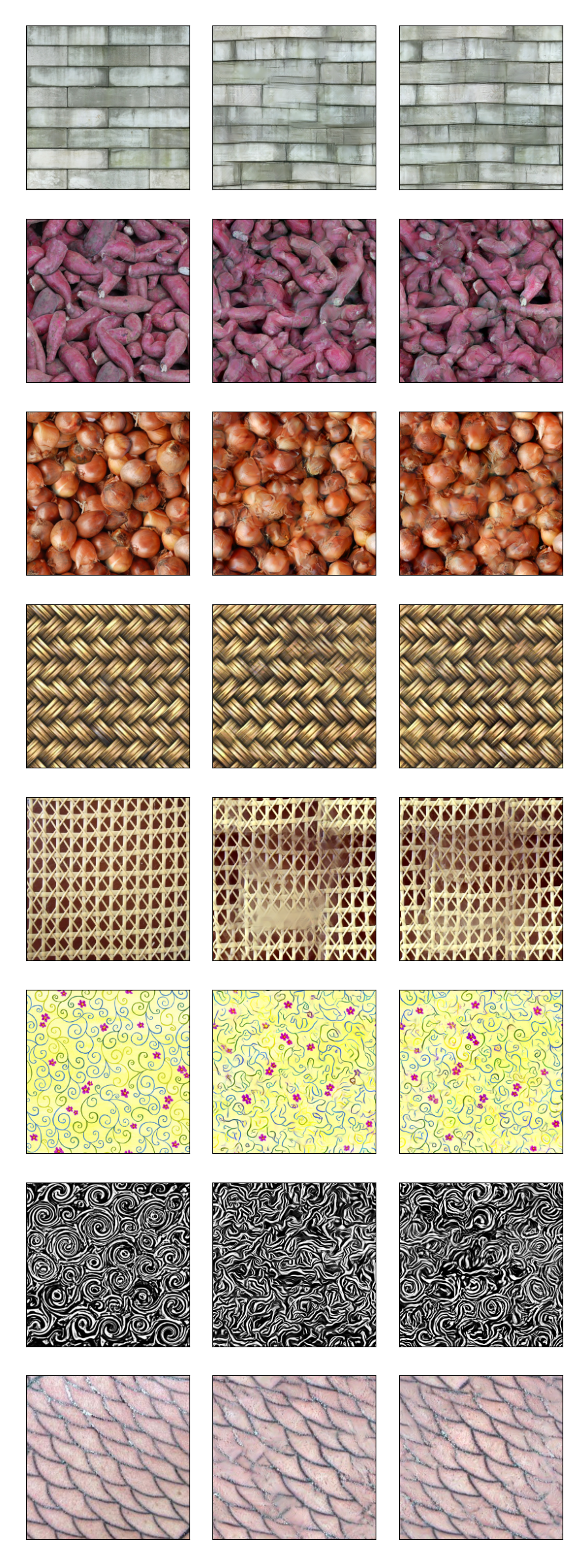}
    \caption{Synthesis results with one layer model and two layer model. \textbf{Left:} Original image. \textbf{Middle:} 1st layer synthesis results. \textbf{Right:} 2nd layer synthesis results. }
    \label{fig:layers}
\end{figure}

\subsection{Methods comparison}
\label{sec: methods comparison}

We use even and odd directional wavelets, along with omnidirectional wavelets as our pre-selected filters in our final model. We also set $J=5$ and use the two-layer model. Our results are compared with \cite{Portilla, vggsyn} in Figure \ref{fig:final}. 

The textures synthesized by our model are generally equivalent to, or superior than, the images synthesized by Portilla and Simoncelli \cite{Portilla}. In fact, while not depicted in Figure \ref{fig:final}, this result holds even if we restrict to one layer, indicating the combination of the ReLU and the selection of even, odd, and omnidirectional filters may provide a more complete statistical description of texture images.

With respect to Gatys \textit{et al.} \cite{vggsyn}, the results are more nuanced. Our results are generally superior for textures with long, rigid edges even though their model is much deeper than our model. Additionally, textures with rigid patterns, but not necessarily long straight lines (e.g., left, last row; right, rows five, eight, nine) also have visually more appealing synthesis results via our method. These results can be attributed to the use of multiscale filters, although, even then, the results of Section \ref{sec: max scale comparison} suggest that such an analysis might be too simplistic. For example, it is possible that a depth three wavelet network with $J=4$ might also achieve similar performance to our current implementation with two layers and $J=5$, which if so would raise questions with respect to other aspects of the VGG network.

Moving on we see that \cite{vggsyn} obtains superior performance for images with long, non-rigid curves, such as the pebbles and onions (left, rows three and nine), fireworks (left, row eight) and swirling type images (right, rows seven and ten). Nevertheless, these results are in line with the observations in Section \ref{sec: layers analysis}, which suggested that these types of textures require depth in order to capture their complex patterns. 

Other images exhibit more subtle differences. For example, in the cracked earth image (left, row four), the synthesized image of \cite{vggsyn} creates a bold effect on the cracks that is not present in the original.
In the crossing image (left, row ten), the background illumination pattern is only correctly preserved by our method. Lastly, the multi-colored dots (right, last row) have a cleaner background with \cite{vggsyn}, but our method creates dots with colors that are not present in the original image, thus showing greater variability. 
Additional numerical results can be found in Appendix \ref{sec: additional numerical results}.

\begin{figure*}
    \centering
    \includegraphics[width = 0.45 \linewidth]{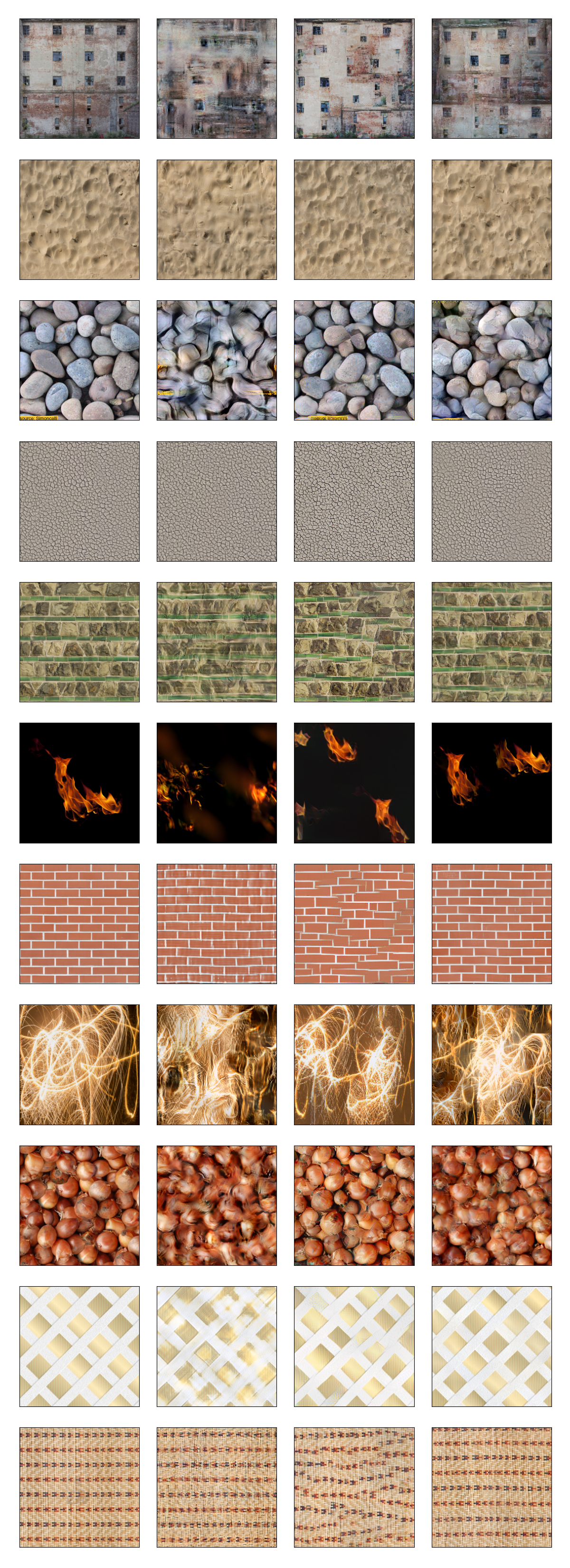}
    \includegraphics[width = 0.45 \linewidth]{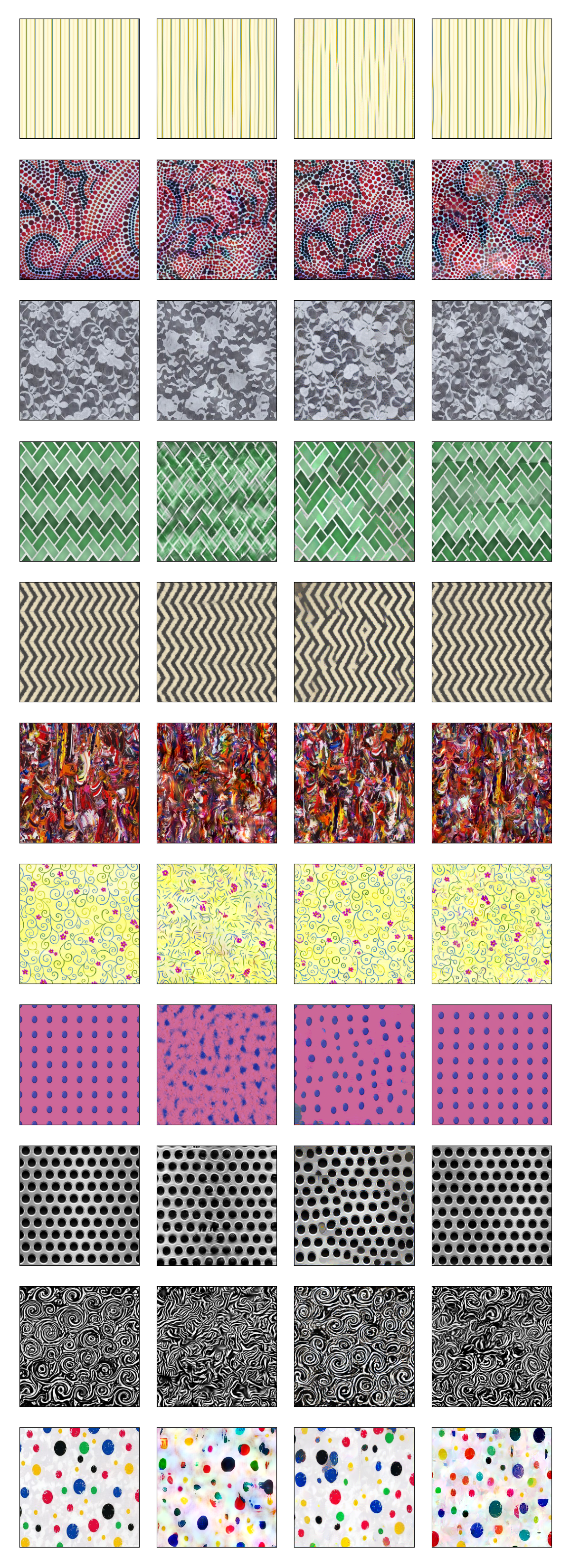}
    \caption{Synthesis results compared to other models. \textbf{Left:} Original images. \textbf{Middle Left:} Results from Portilla and Simoncelli \cite{Portilla}. \textbf{Middle Right: } Results from Gatys \textit{et al.} \cite{vggsyn}. \textbf{Right:} Results from our two layer model.}
    \label{fig:final}
\end{figure*}

\section{Conclusion}
\label{sec: conclusion}

We presented a unique texture synthesis algorithm that melds aspects of Portilla and Simoncelli \cite{Portilla}, Gatys \textit{et al.} \cite{vggsyn}, and Mallat \cite{scattering}, while also incorporating new ideas on filter design, multi-layer structure, and the invertibility of CNNs. Our numerical analysis provides insight into the workings of statistics-based texture synthesis algorithms. Synthesized textures are competitive with the state-of-the-art and in some cases superior to \cite{vggsyn}, thus providing a potential alternative. Nevertheless, issues such as the trade-off between network depth and filter scale are not fully resolved, and invite future research endeavors.

\appendices

\section{Proofs of theorems}
\label{sec: proofs of theorems}

In this section we prove Theorems 1 and 2 from the main paper. 

\subsection{Analysis with wavelets on Diracs and jump functions}
\label{sec: analysis of jump discontinuity and dirac}

\thmMorlet*

Figure \ref{fig:1d_wavelets} shows examples of the wavelets and corresponding Fourier transforms. 

\begin{figure}[ht]
    \centering
    \includegraphics[width = 0.95 \linewidth]{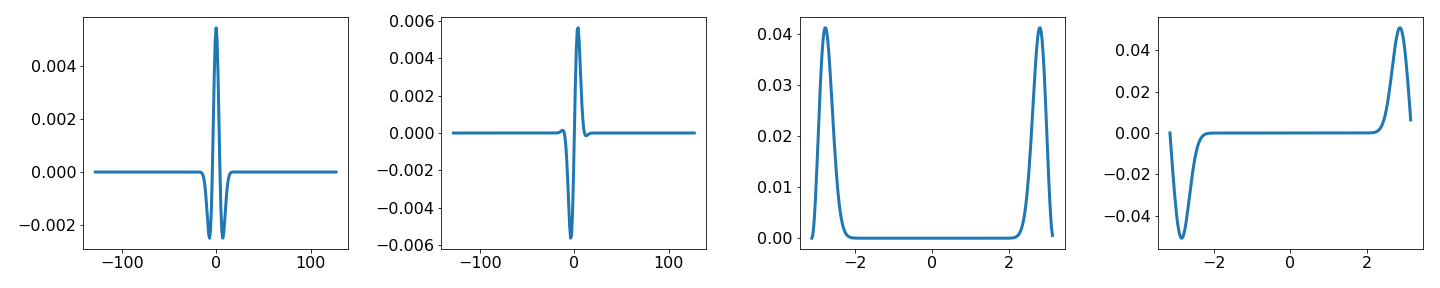}
    \caption{From left to right: 1D even wavelet, 1D odd wavelet, FFT (real part) of even wavelet, FFT (imagery part) of odd wavelet.}
    \label{fig:1d_wavelets}
\end{figure}

\begin{lemma}
\label{lemma: odd function}
Let $f(u)$ be an odd function, we have 
$$\int_{\R} \sigma(\gamma \cdot f) \sigma(\gamma' \cdot f) = \int_{\R} \sigma(-\gamma \cdot f) \sigma(-\gamma' \cdot f)$$
for $\gamma, \gamma' \in \{-1,+1\}$.
\end{lemma}

\subsubsection{Diracs}

In this section we prove Theorem \ref{thm: necessity of morlet} for the 1D Dirac function $y_1(t)$. First the wavelet transform of $y_1$ is:
\begin{equation*}
    y_1 \ast \psi^o (u) = \psi^o (u), \quad y_1 \ast \psi^e (u) = \psi^e (u)
\end{equation*}
Figure \ref{fig:diracs} shows the wavelet transforms for $y_1$ and $-y_1$. Since $\psi^o(u)$ is an odd function, with Lemma \ref{lemma: odd function} we have $\cov_{y_1}^o = \cov_{-y_1}^o$. However since $\psi^e(u)$ is an even function, generally we have $\cov_{y_1}^{e} (+1, +1) \neq \cov_{-y_1}^{e} (+1, +1)$. Therefore $\cov_{y_1}^e \neq \cov_{-y_1}^e$. Figure \ref{fig:dirac_cov} verifies this conclusion numerically. 

\begin{figure}[ht]
    \centering
    \includegraphics[width = 0.95 \linewidth]{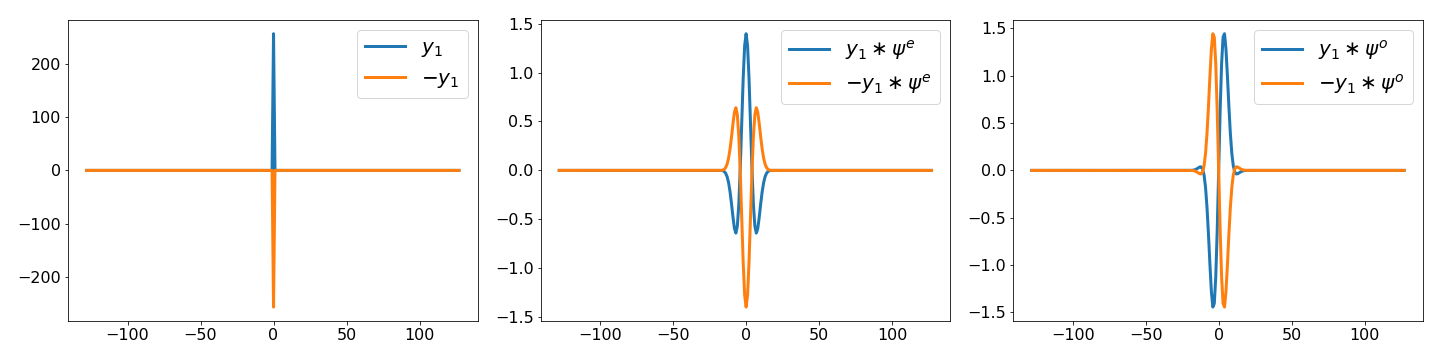}
    \caption{\textbf{Left}: Two Dirac functions $y_1$ and $-y_1$. \textbf{Middle}: Wavelet coefficients with the even wavelet. \textbf{Right}: Wavelet coefficients with the odd wavelet. }
    \label{fig:diracs}
\end{figure}

\begin{figure}[ht]
    \centering
    \includegraphics[width = 0.95 \linewidth]{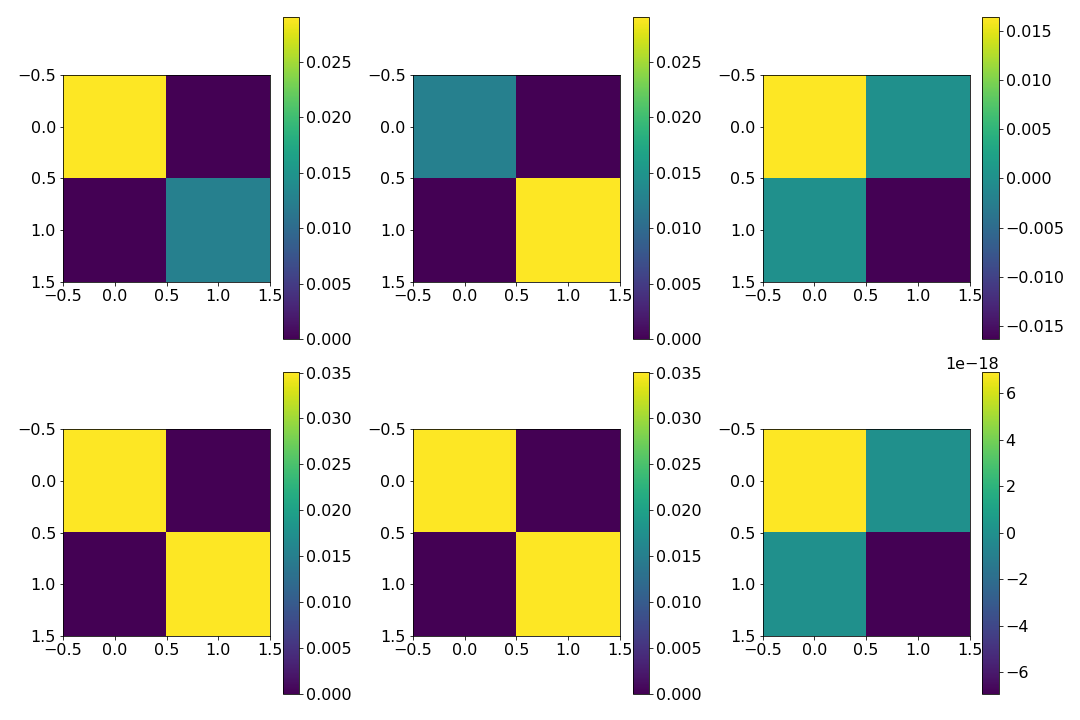}
    \caption{Upper row from left to right: $\cov_{y_1}^e$, $\cov_{-y_1}^e$, $\cov_{y_1}^e - \cov_{-y_1}^e$.  Lower row from left to right: $\cov_{y_1}^o$, $\cov_{-y_1}^o$, $\cov_{y_1}^o - \cov_{-y_1}^o$. This numerically verified that even wavelet is able to distinguish the two dirac functions from the covariance statistics while odd wavelet cannot. }
    \label{fig:dirac_cov}
\end{figure}

\subsubsection{Jump discontinuity}

In this section we prove Theorem \ref{thm: necessity of morlet} for the jump function $y_2(t)$. The wavelet transforms satisfy:
\begin{equation}
    y_2 \ast \psi^{\beta} (u) = \int_{\R} y_2(u-t)\psi^{\beta}(t) dt = \int_{-\infty}^{u} \psi^{\beta}(t) dt
\end{equation}
for $\beta \in \{e, o\}$. Figure \ref{fig:jumps} illustrates the convolution of the even and odd wavelet with $y_2$ and $-y_2$.

\begin{remark}
Since $\psi^e$ is an integrable even function, then $f^e(u) = \int_{-\infty}^{u} \psi^e(t) dt$ is an odd function.
\end{remark}

\begin{remark}
Since $\psi^o$ is an integrable odd function, then $f^o(u) = \int_{-\infty}^{u} \psi^o(t) dt$ is an even function.
\end{remark}
Therefore, with Lemma \ref{lemma: odd function} we have $\cov_{y_2}^e = \cov_{-y_2}^e$. Generally we also have $\cov_{y_2}^o \neq \cov_{-y_2}^o$. 
Figure \ref{fig:jumps_cov} gives the numerical verification. 

\begin{figure}
    \centering
    \includegraphics[width = 0.95 \linewidth]{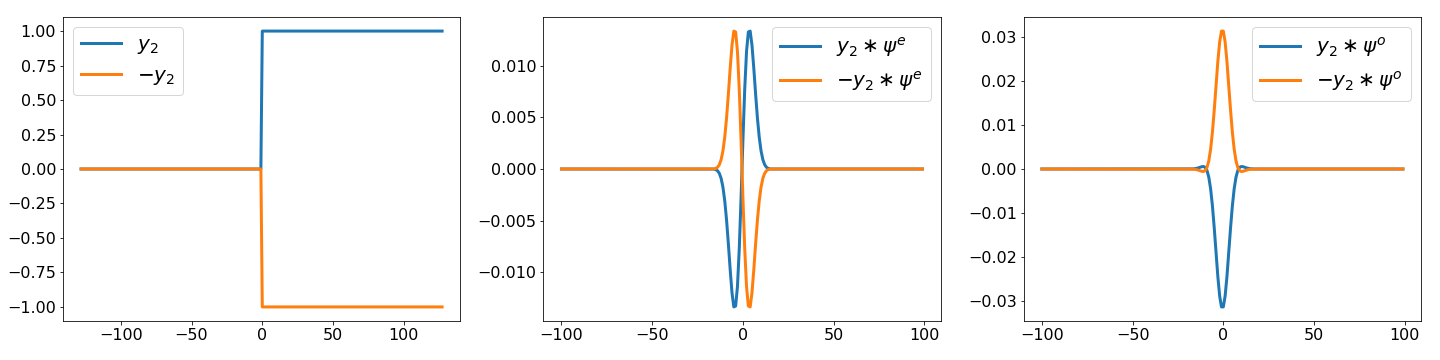}
    \caption{\textbf{Left}: Two jump functions $y_2$ and $-y_2$. \textbf{Middle}: Wavelet coefficients with the even wavelet. \textbf{Right}: Wavelet coefficients with the odd wavelet. }
    \label{fig:jumps}
\end{figure}

\begin{figure}
    \centering
    \includegraphics[width = 0.95 \linewidth]{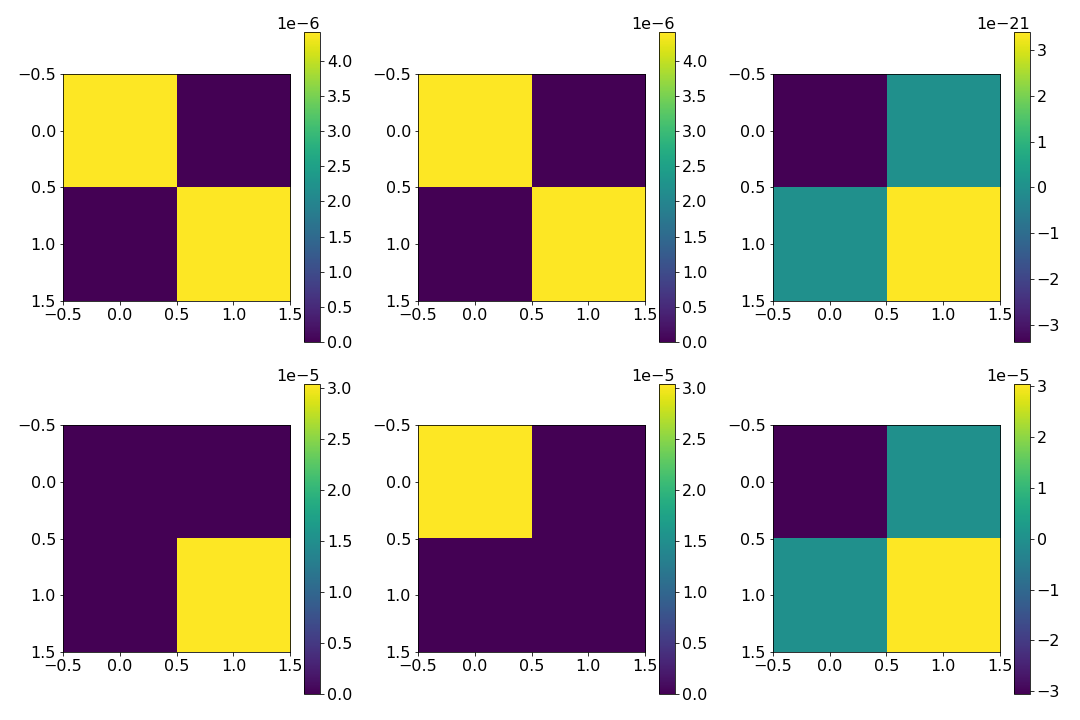}
    \caption{Upper row from left to right: $\cov_{y_2}^e$, $\cov_{-y_2}^e$, $\cov_{y_2}^e - \cov_{-y_2}^e$.  Lower row from left to right: $\cov_{y_2}^o$, $\cov_{-y_2}^o$, $\cov_{y_2}^o - \cov_{-y_2}^o$. This numerically verified that odd wavelet is able to distinguish the two jump functions from the covariance statistics while even wavelet cannot. }
    \label{fig:jumps_cov}
\end{figure}

\subsection{Invertibility of 2nd layer}
\label{sec: invertibility of second layer}

Now we prove Theorem \ref{thm: 2nd layer inverse}. We will need the following lemma.

\begin{lemma} \label{lemma: LP frame}
If $\{ \phi_J, h, \psi_{j, \theta}^o : 0 \leq j < J, \enspace \theta \in \Theta_M \}$ forms a frame and if $\widehat{g}$ is non-negative, radial, and a decreasing function of $|\omega|$, then $\{ \phi_J, h, \sum_{j=0}^{J-1} \psi_{j, \theta}^o : \enspace \theta \in \Theta_M \}$ also forms a frame. 
\end{lemma}
\begin{proof}
With the definition of frame, we know there exist two constants $0 < A \leq B < \infty$ such that:
\begin{align*}
    & A \leq |\widehat{\phi}_J(\omega)|^2 + |\widehat{h}(\omega)|^2 + \sum_{j, \theta} |\widehat{\psi}_{j, \theta}^o(\omega)|^2 \leq B, \\
    & \forall \, \omega \in \Omega \cap [-\pi, \pi]^2 \, .
\end{align*}
we need to prove there exist two constants $0 < A' \leq B' < \infty$ such that:
\begin{align*}
    & A' \leq |\widehat{\phi}_J(\omega)|^2 + |\widehat{h}(\omega)|^2 + \sum_{\theta} |\sum_j \widehat{\psi}_{j, \theta}^o(\omega)|^2 \leq B', \\
    & \forall \, \omega \in \Omega \cap [-\pi, \pi]^2 \, .
\end{align*}
The upper bound always exists as long as we have a finite number of filters. Therefore we only prove the lower bound. The key point is to prove:
\begin{equation}
\label{eqn: sum of odd}
\sum_j |\widehat{\psi}_{j, \theta}^o(\omega)|^2 \leq |\sum_j \widehat{\psi}_{j, \theta}^o(\omega)|^2, \forall \theta \in \Theta_M 
\end{equation}
Without loss of generosity, we set $\theta = 0$ and omit this notation in the following proof. Recall the odd directional wavelet $\psi^o(u) = g(u) \sin (\xi \cdot u)$, which has Fourier transform:
\begin{equation}
\label{eqn: odd fft}
\begin{aligned}
    \widehat{\psi}^o(\omega) &= \frac{\widehat{g}_{\sigma}(\omega - \xi) - \widehat{g}_{\sigma}(\omega + \xi)}{2i}
\end{aligned}
\end{equation}
Bringing equation \eqref{eqn: odd fft} into equation \eqref{eqn: sum of odd}, we need to prove:
\begin{equation}
\label{eqn: odd sum g}
\begin{aligned}
    & \sum_j |\widehat{g}_{\sigma, j}(\omega - \xi) - \widehat{g}_{\sigma, j}(\omega + \xi)|^2 \\
    \leq &|\sum_j \widehat{g}_{\sigma, j}(\omega - \xi) - \widehat{g}_{\sigma, j}(\omega + \xi)|^2
\end{aligned}
\end{equation}
If $\widehat{g}$ is non-negative, radial, and a decreasing function of $|\omega|$, one can prove: 
\begin{itemize}
    \item If $|\omega - \xi| < |\omega + \xi|$, then $\widehat{g}_{\sigma}(\omega - \xi) - \widehat{g}_{\sigma}(\omega + \xi) \geq 0$. For any such $\omega$ we also have $|2^{-j}\omega - \xi| < |2^{-j}\omega + \xi|$, and $\widehat{g}_{\sigma}(2^{-j}\omega - \xi) - \widehat{g}_{\sigma}(2^{-j}\omega + \xi) \geq 0$.
    
    \item If $|\omega - \xi| > |\omega + \xi|$, then $\widehat{g}_{\sigma}(\omega - \xi) - \widehat{g}_{\sigma}(\omega + \xi) \leq 0$. For any such $\omega$ we also have $|2^{-j}\omega - \xi| > |2^{-j}\omega + \xi|$, and $\widehat{g}_{\sigma}(2^{-j}\omega - \xi) - \widehat{g}_{\sigma}(2^{-j}\omega + \xi) \leq 0$.
    
    \item If $|\omega - \xi| = |\omega + \xi|$, then $\widehat{g}_{\sigma}(\omega - \xi) - \widehat{g}_{\sigma}(\omega + \xi) = 0$. For any such $\omega$ we also have $|2^{-j}\omega - \xi| = |2^{-j}\omega + \xi|$, and $\widehat{g}_{\sigma}(2^{-j}\omega - \xi) - \widehat{g}_{\sigma}(2^{-j}\omega + \xi) = 0$.
\end{itemize}
Then for all $\omega$, $\widehat{g}_{\sigma}(\omega - \xi) - \widehat{g}_{\sigma}(\omega + \xi)$ and $\widehat{g}_{\sigma, j}(\omega - \xi) - \widehat{g}_{\sigma, j}(\omega + \xi)$ have the same sign for all $j$ (either non-positive or non-negative). One can prove: $(\sum_i a_i)^2 \geq \sum_i a_i^2$ if all the $a_i$ are non-negative or non-positive. Therefore, equation \eqref{eqn: odd sum g} is true, and $|\sum_j \widehat{\psi}_{j,\theta}^o (\omega)|^2 \geq \sum_j |\widehat{\psi}_{j,\theta}^o (\omega)|^2$ for all $\theta$. The lemma is proved. 
\end{proof}

\thmsecondlayerinverse*

\begin{proof}
If $\{ \phi_J, h, \psi_{j, \theta}^o : 0 \leq j < J, \enspace \theta \in \Theta_M \}$ forms a frame, then 
\begin{align*}
    & \{ \phi_J \, , \, h \, , \,\psi_{j, \alpha}^{\beta} : 0 \leq j < J, (\alpha, \beta) \in \{ (\theta, e), (\theta, o), (\ell, p) \}, \\ 
    &\enspace \theta \in \Theta_M, \enspace 0 \leq \ell < L \} \, .
\end{align*}
also forms a frame, i.e., 
\begin{equation}
    x = x \ast \phi_J \ast \widetilde{\phi_J} + x \ast h \ast \widetilde{h} + \sum_{j=0}^{J-1} \sum_{\alpha, \beta} x \ast \psi_{j, \alpha}^{\beta} \ast \widetilde{\psi_{j, \alpha}^{\beta}}, \forall x
\end{equation}
Therefore we can reconstruct $\sum_{j=0}^{J-1} \sigma (\gamma \cdot x \ast \psi_{j, \alpha}^{\beta})$ from $U_J^2x$. With $t = \sigma(t) - \sigma(-t)$, we are able to get $\sum_{j=0}^{J-1} x \ast \psi_{j, \alpha}^{\beta}$, which can also be written as $x \ast \sum_{j=0}^{J-1}\psi_{j, \alpha}^{\beta}$. At this point, we have the following updated responses
\begin{align*}
    \{ x \ast \phi_J, x \ast h, x \ast \sum_{j=0}^{J-1} \psi_{j, \alpha}^{\beta}, (\alpha, \beta) \in \{ (\theta, e), (\theta, o), (\ell, p) \} \}
\end{align*}
With Lemma \ref{lemma: LP frame}, $\{ \phi_J, h, \sum_{j=0}^{J-1} \psi_{j, \theta}^o\}$ forms a frame, then $\{ \phi_J, h, \sum_{j=0}^{J-1} \psi_{j, \alpha}^{\beta}\}$ also forms a frame. Therefore we can reconstruct the image $x$ from the above responses: 
\begin{equation}
    x = x \ast \phi_J \ast \widetilde{\phi_J} + x \ast h \ast \widetilde{h} + \sum_{\alpha, \beta} x \ast \sum_{j=0}^{J-1} \psi_{j, \alpha}^{\beta} \ast \widetilde{\sum_{j=0}^{J-1} \psi_{j, \alpha}^{\beta}}
\end{equation}
\end{proof}

\section{Implementation details}
\label{sec: implemetation details}

\subsection{Wavelets}

\begin{figure}
    \centering
    \includegraphics[width = 0.45 \linewidth]{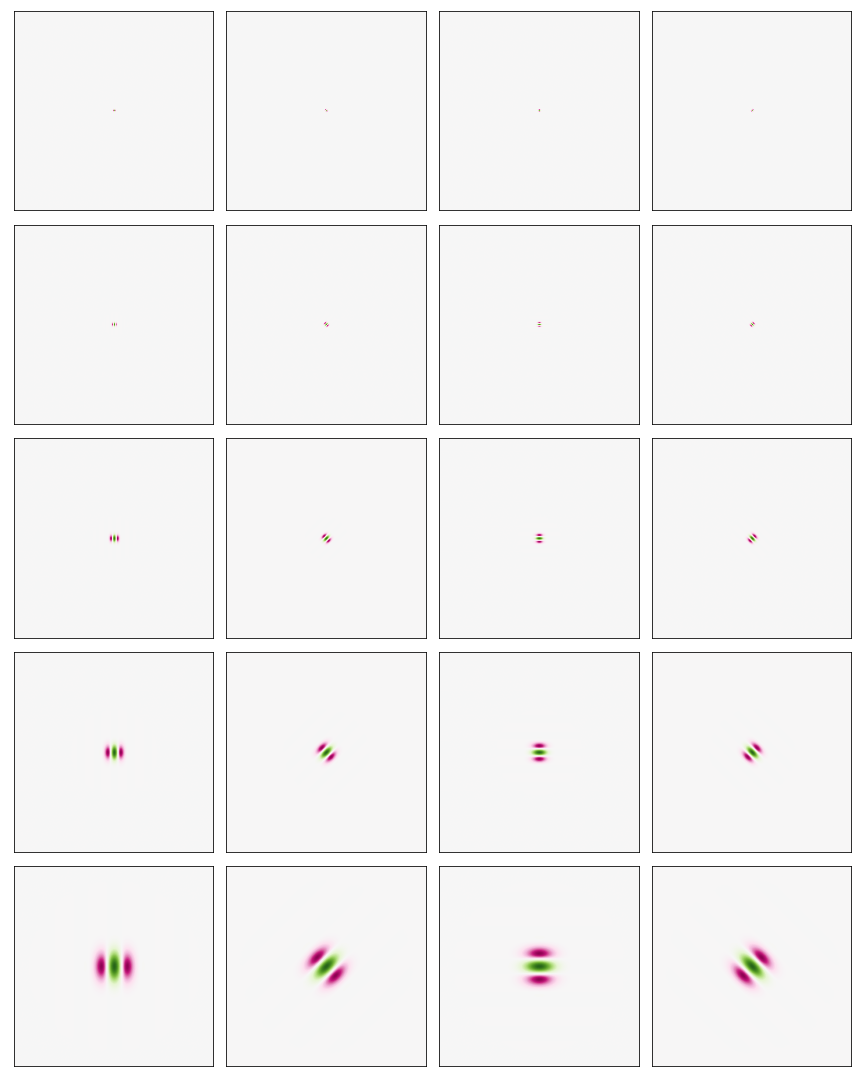}
    \includegraphics[width = 0.45 \linewidth]{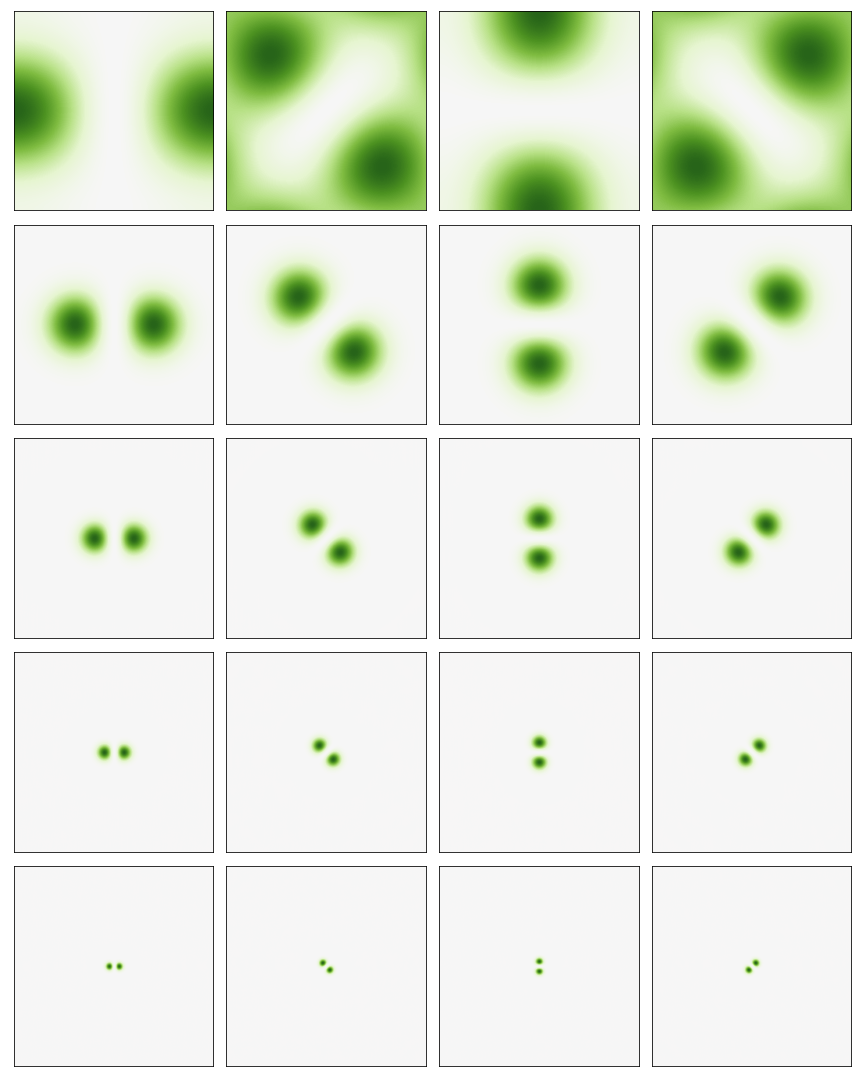}
    
    \includegraphics[width = 0.45 \linewidth]{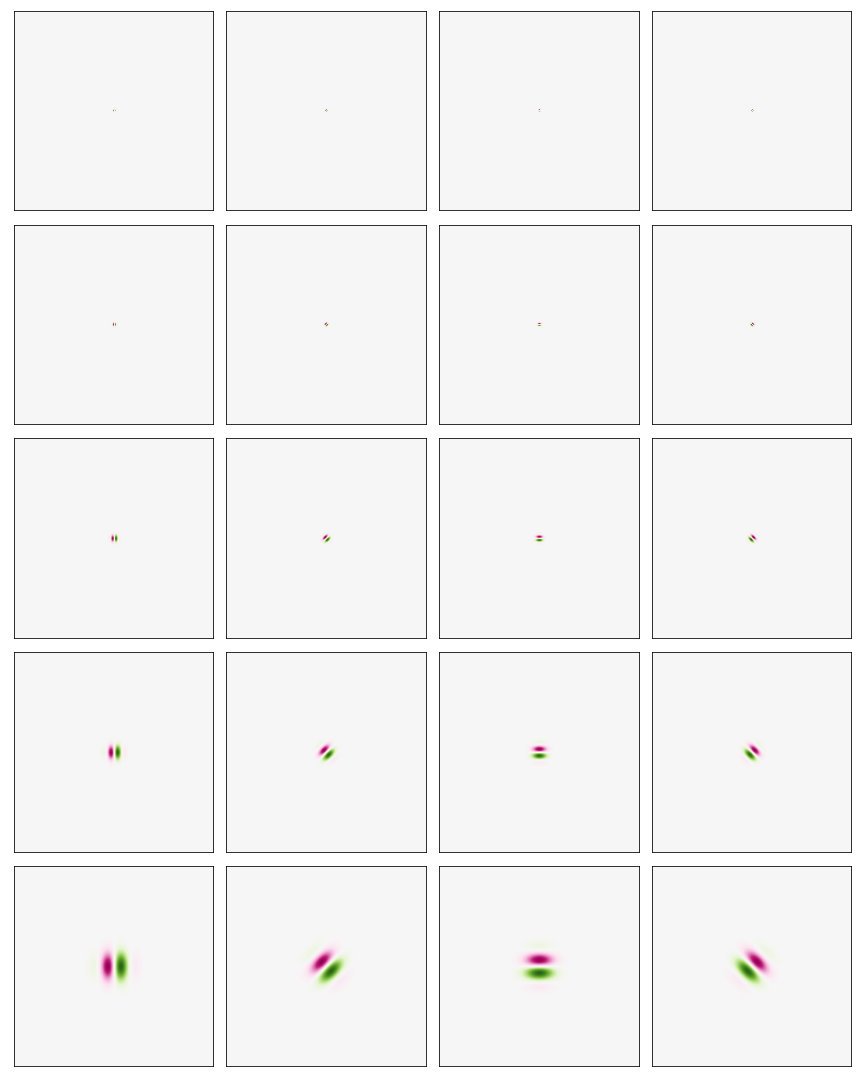}
    \includegraphics[width = 0.45 \linewidth]{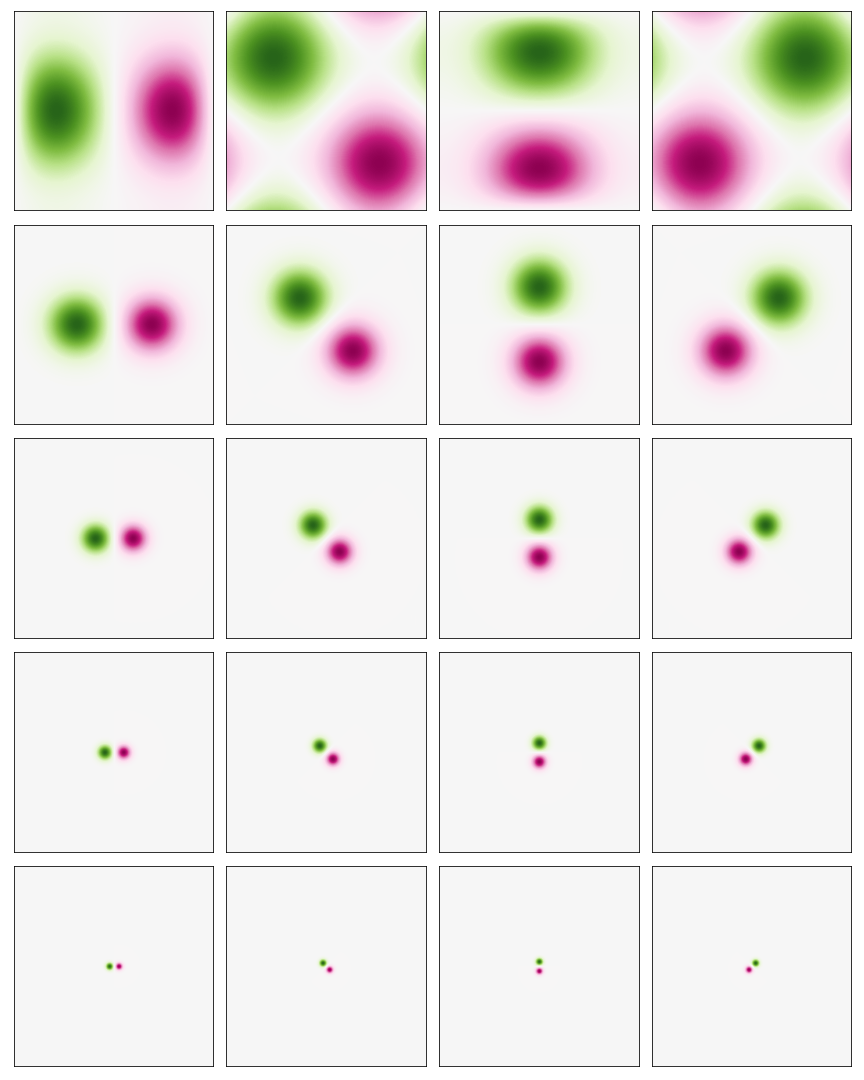}

    \includegraphics[width = 0.45 \linewidth]{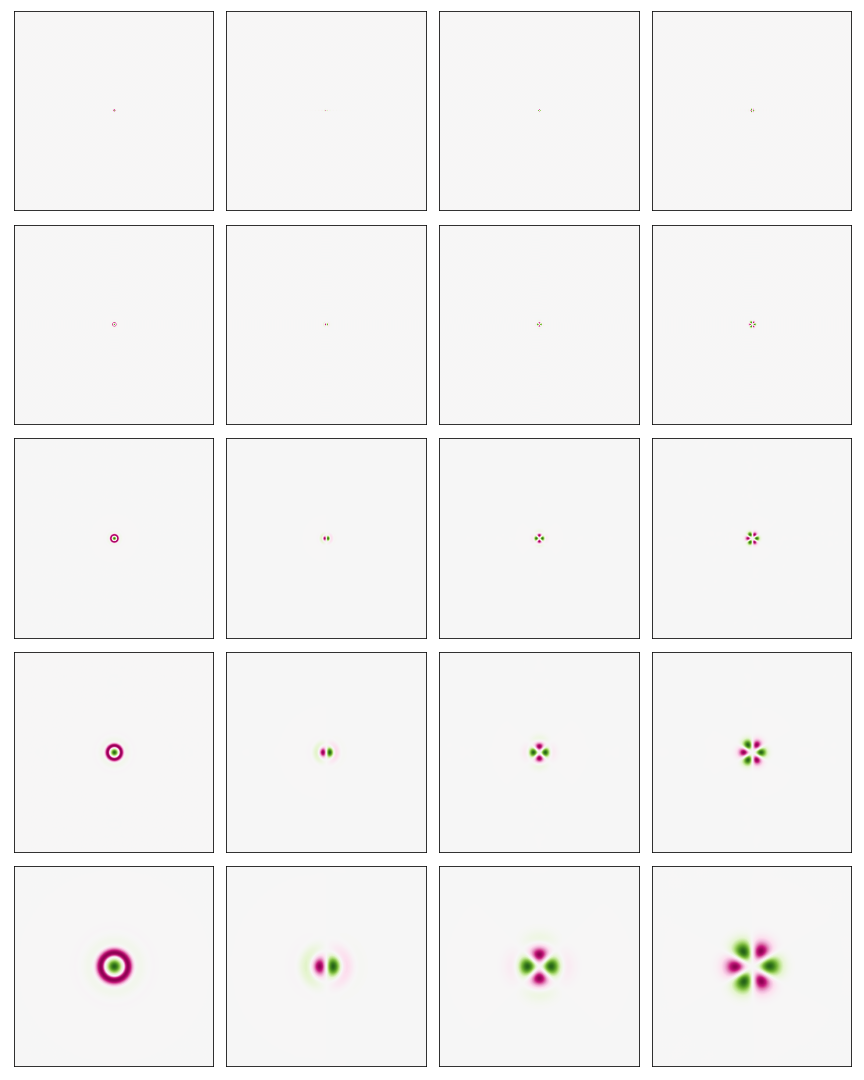}
    \includegraphics[width = 0.45 \linewidth]{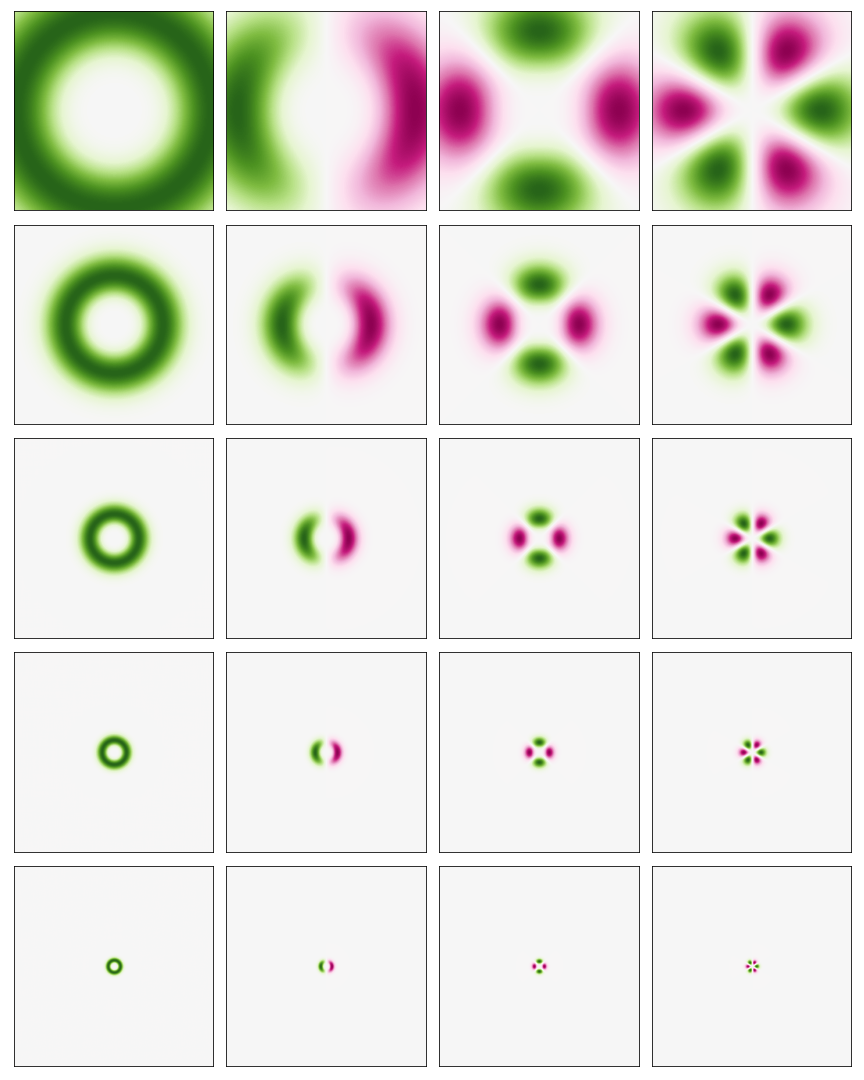}
    \caption{\textbf{Upper}: Even directional wavelets in space and frequency (FFT). \textbf{Middle}: Odd directional wavelets in space and frequency (FFT). \textbf{Lower}:  Omnidirectional wavelets in space and frequency (FFT). Each block shows the wavelet family with different scales and oscillations. }
    \label{fig: wavelets}
\end{figure}

We show the wavelet families in Figure \ref{fig: wavelets}.  Recall:
\begin{align*}
    \psi^e (u) &:= g(u) \cos (\xi \cdot u) \, , \\
    \psi^o (u) &:= g(u) \sin (\xi \cdot u) \, , \\
    \psi_{\ell}^p (u) &:= a_{\ell}(u) \cos (\ell \varphi) \, ,
\end{align*}
In particular for directional wavelets, we use $g = g_\sigma$ as a gaussian function with variance $\sigma^2$. 

For the three wavelets, they all have local support both in space and frequency. As $j$ increase from 0 to $J-1$ (from top to bottom in each block), the wavelet has larger support in space and smaller support in frequency. For directional wavelets, the wavelet support varies in directions with rotations (from left to right in each block) to capture directional oscillations in images. For omnidirectional wavelets, the wavelet either oscillates radially or angularly or both. The total number of oscillations is fixed across the four wavelets. With all of the wavelets described above, along with the low pass and high pass, a frame is formed and the whole frequency field is covered. 

\subsection{Reduction of second layer statistics}
Recall at the second layer we compute:
\begin{align*}
    U_J^2 &x := \Bigg\{ U_J^1 \left( \sum_{j=0}^{J-1} \sigma (\gamma \cdot x \ast \psi_{j, \alpha}^{\beta}) \right) : \\
    &(\alpha, \beta) \in \{ (\theta, e), (\theta, o), (\ell, p) \}, \enspace \theta \in \Theta_M, \enspace 0 \leq \ell < L \Bigg\} \, .
\end{align*}
that is:
\begin{align*}
    U_J^2 x := \Bigg\{ &\sum_{j_1=0}^{J-1} \sigma (\gamma_1 \cdot x \ast \psi_{j_1, \alpha_1}^{\beta_1})  \ast \phi_J, \\
    & \sum_{j_1=0}^{J-1} \sigma (\gamma_1 \cdot x \ast \psi_{j_1, \alpha_1}^{\beta_1})  \ast h, \\
    & \sigma \left( \gamma_2 \cdot \left( \sum_{j_1=0}^{J-1} \sigma (\gamma_1 \cdot x \ast \psi_{j_1, \alpha_1}^{\beta_1}) \right) \ast \psi_{j_2, \alpha_2}^{\beta_2} \right): \\
    &(\alpha_1, \beta_1), (\alpha_2, \beta_2) \in \{ (\theta, e), (\theta, o), (\ell, p) \}, \\
    & \gamma_1, \gamma_2 \in \{ +1, -1 \} \Bigg\} \, .
\end{align*}
In particular for the third item, we apply another layer of the wavelet transform to the first layer responses. In numerical experiments for directional wavelets, we find $\sum_{j_1=0}^{J-1} \sigma (\gamma_1 \cdot x \ast \psi_{j_1, \theta_1}^{\beta_1})$ is essentially supported at the direction $\theta = \theta_1$. Therefore to reduce the number of total statistics and save computation, we set $\theta_2 = \theta_1$, i.e., the second layer of directional wavelets has the same direction as the first layer directional responses. We also add a residual wavelet to keep track of the residual frequencies and match the variance. We numerically verified with this restriction, there is little loss in the image quality. 

\subsection{Matching of second layer statistics}

When we use the second layer statistics to synthesize texture images, we initialize the image with the first layer result. This is also equivalent to matching the first layer statistics until convergence, then matching both first layer and second layer statistics. In practice, we see the synthesized image has better quality than matching both first and second layer statistics from noise, i.e., matching both from the beginning. 

Figure \ref{fig: noise vs 1st layer} compares the synthesized images using these two strategies. In particular, the middle right column shows synthesized images initialized from first layer result and the right column shows synthesized images initialized from noise. We notice matching only first layer statistics significantly reduces the first layer loss and the learned images are already well structured. Initialized from such images, second layer statistics refine the details. On the other hand, initializing from noise fails to reduce the first layer loss to a very small value. The learned images generally lose large structures and are not as good as results from the other strategy.

\begin{figure}
    \centering
    \includegraphics[width = 0.95 \linewidth]{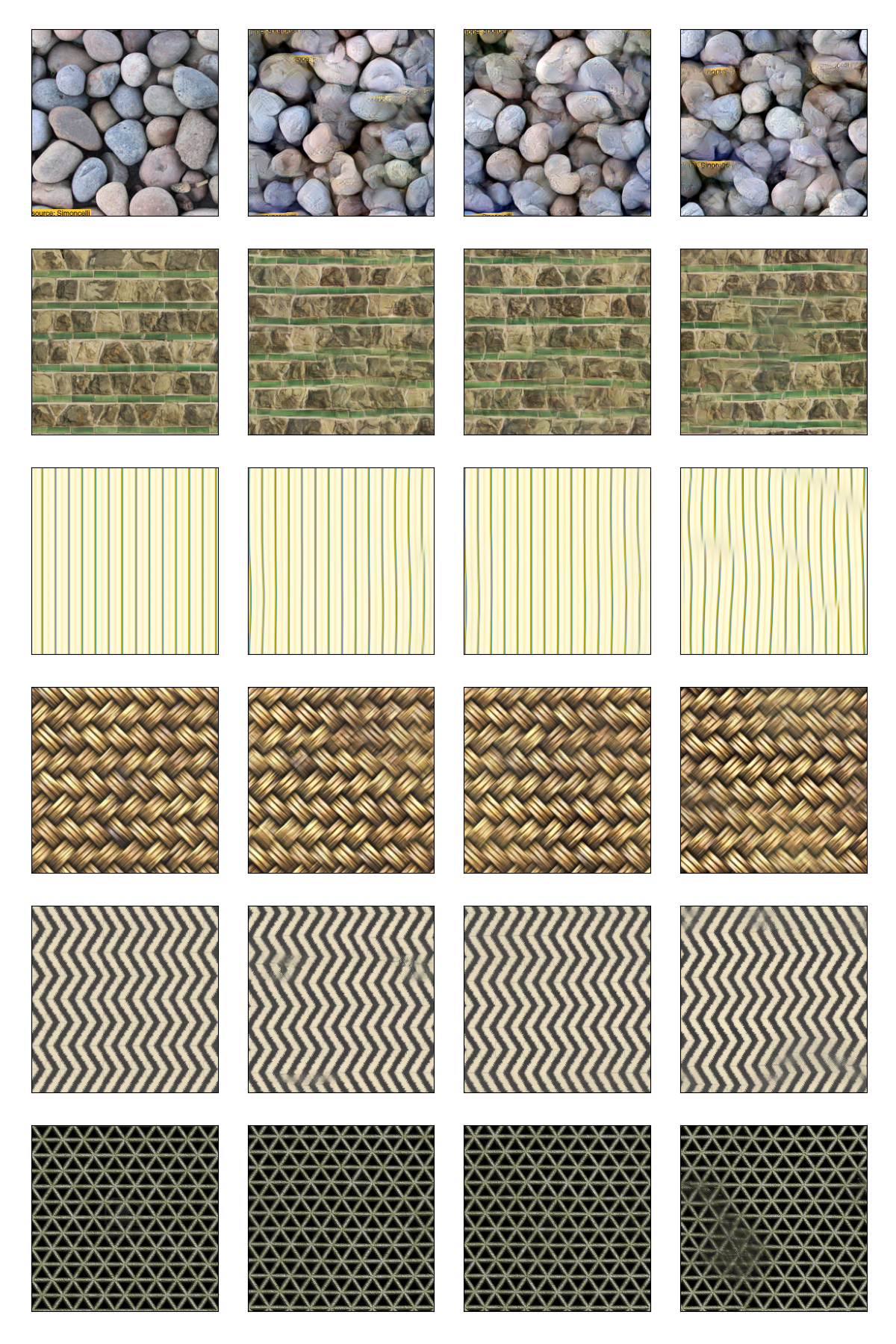}
    \caption{\textbf{Left:} Original images. \textbf{Middle left:} Images synthesized from first layer. \textbf{Middle right:} Images synthesized from second layer, initialized from first layer result. \textbf{Right: } Images synthesized from second layer, initialized from uniform noise. }
    \label{fig: noise vs 1st layer}
\end{figure}

\subsection{Analysis of number of iterations}

In this section we discuss the relationship among the number of iterations, loss, and image quality. Figures \ref{fig: images 1}, \ref{fig: images 2}, \ref{fig: images 3} show the synthesis process of different textures. For all of these tests, we run second layer synthesis and run for 600 iterations. 
We plot the logarithm of the relative loss in these plots, i.e., $\log_{10}(\text{loss})$ where $\text{loss} = \frac{\| S_J^i x - S_J^i x^{\star}\|}{\|S_J^i x\|}$ for $i = 1,2$, $x$ is the reference image, and $x^{\star}$ is the synthesized image. 

Figure \ref{fig: images 1} shows the synthesis process of micro-textures plus flowers. We start from noise to match the second layer to simplify the analysis. After iteration 0, the images already look much better than noise, although the colors are not fully matched.  At iteration 40, the second layer loss is reduced to approximately $10^{-3.5}$ and the synthesized images are of high quality. As we continue the algorithm, at iteration 100 and 590, the images smoothly shift from one sample to another sample which are from the same texture class and the loss barely changed. In general, for such micro-textures, the algorithm converges fast and the loss is reduced long before we stop.

\begin{figure}
    \centering
    \includegraphics[width = 0.95 \linewidth]{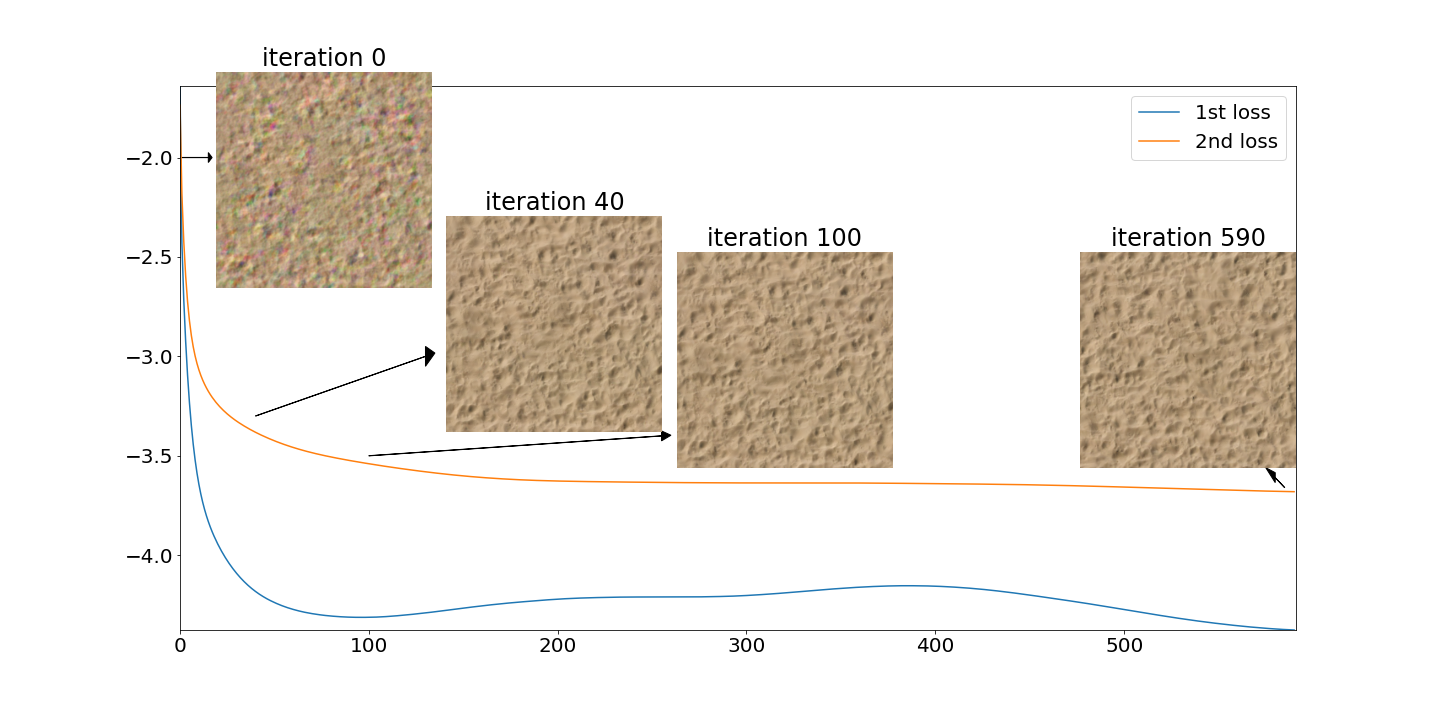}
    \includegraphics[width = 0.95 \linewidth]{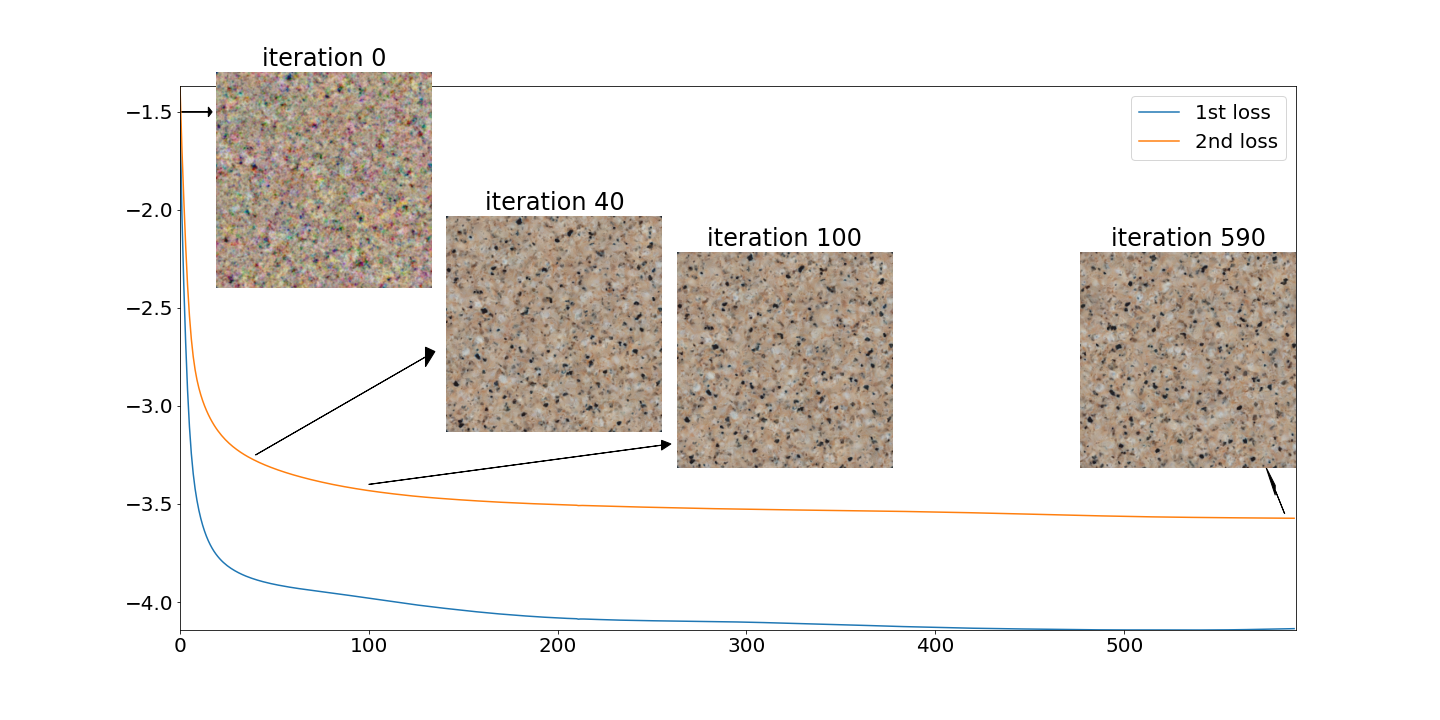}
    \includegraphics[width = 0.95 \linewidth]{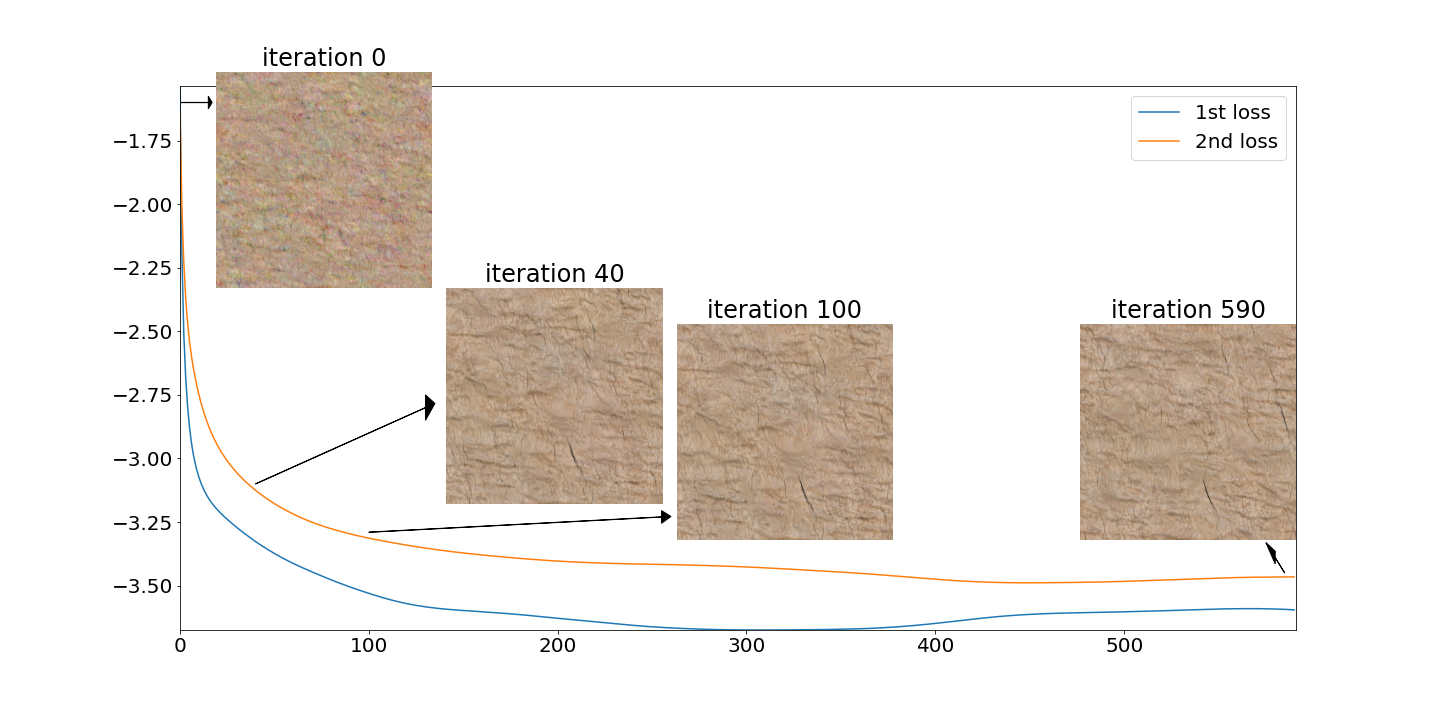}
    \includegraphics[width = 0.95 \linewidth]{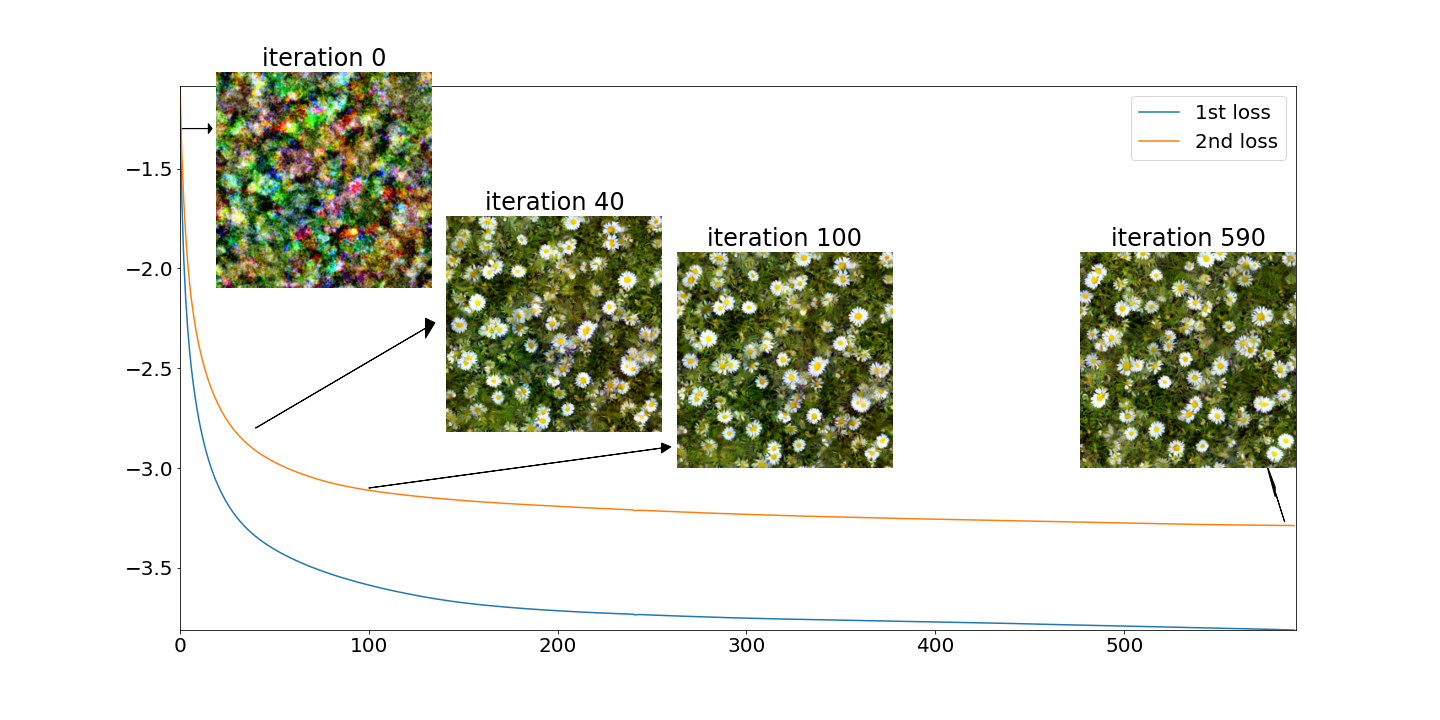}
    \caption{The synthesis process for different micro-textures plus flowers.}
    \label{fig: images 1}
\end{figure}

Figure \ref{fig: images 2} shows the synthesis process for textures that have more structure. Still we start from noise. At iteration 0, the colors are mismatched, similar to Figure \ref{fig: images 1}. At iteration 40 and 100, the colors are matched and general structures are learned, but the holes and frames are not aligned perfectly. And the relative loss is not reduced. Finally at iteration 590, the second layer loss is reduced to around $10^{-4}$ and the synthesized images are of good quality. The algorithm on these types of images generally takes longer to converge than for micro-textures. As the number of iterations increases, the image quality is continuously improved.

\begin{figure}
    \centering
    \includegraphics[width = 0.95 \linewidth]{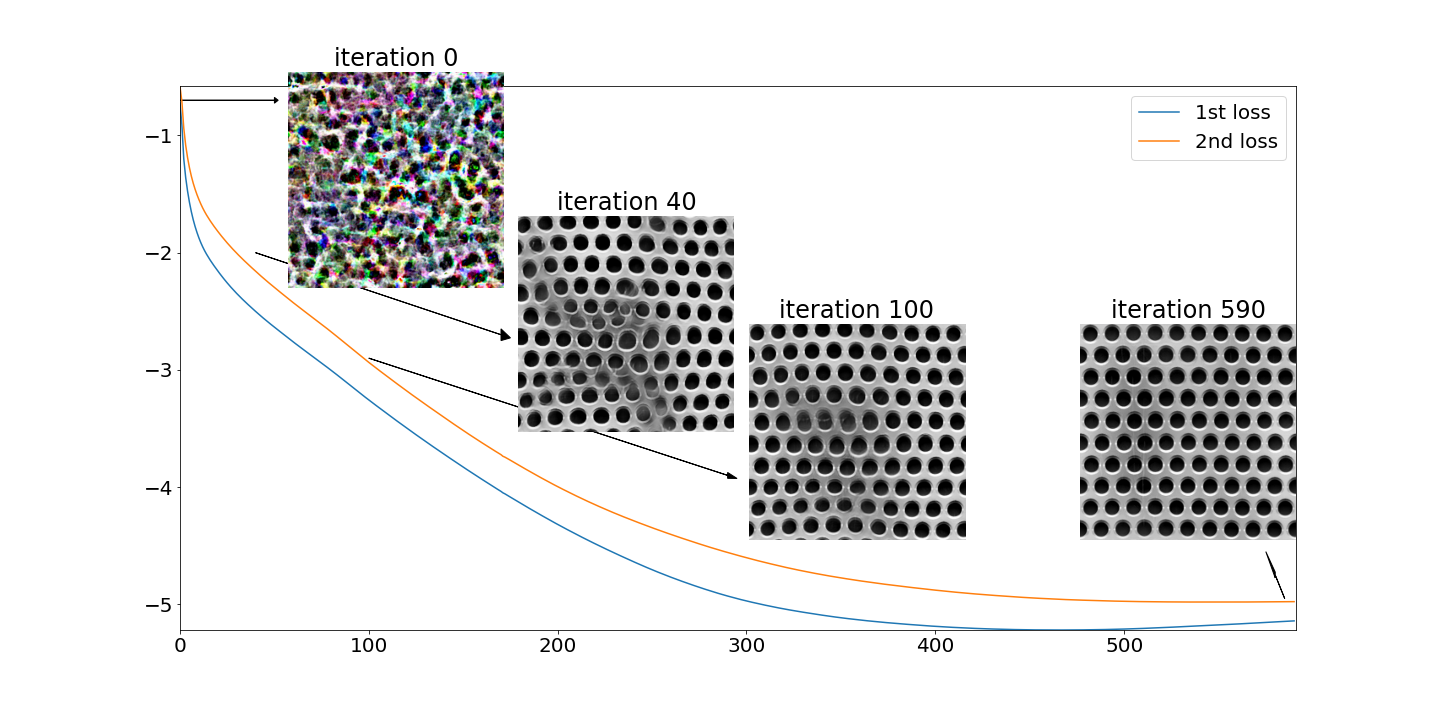}
    \includegraphics[width = 0.95 \linewidth]{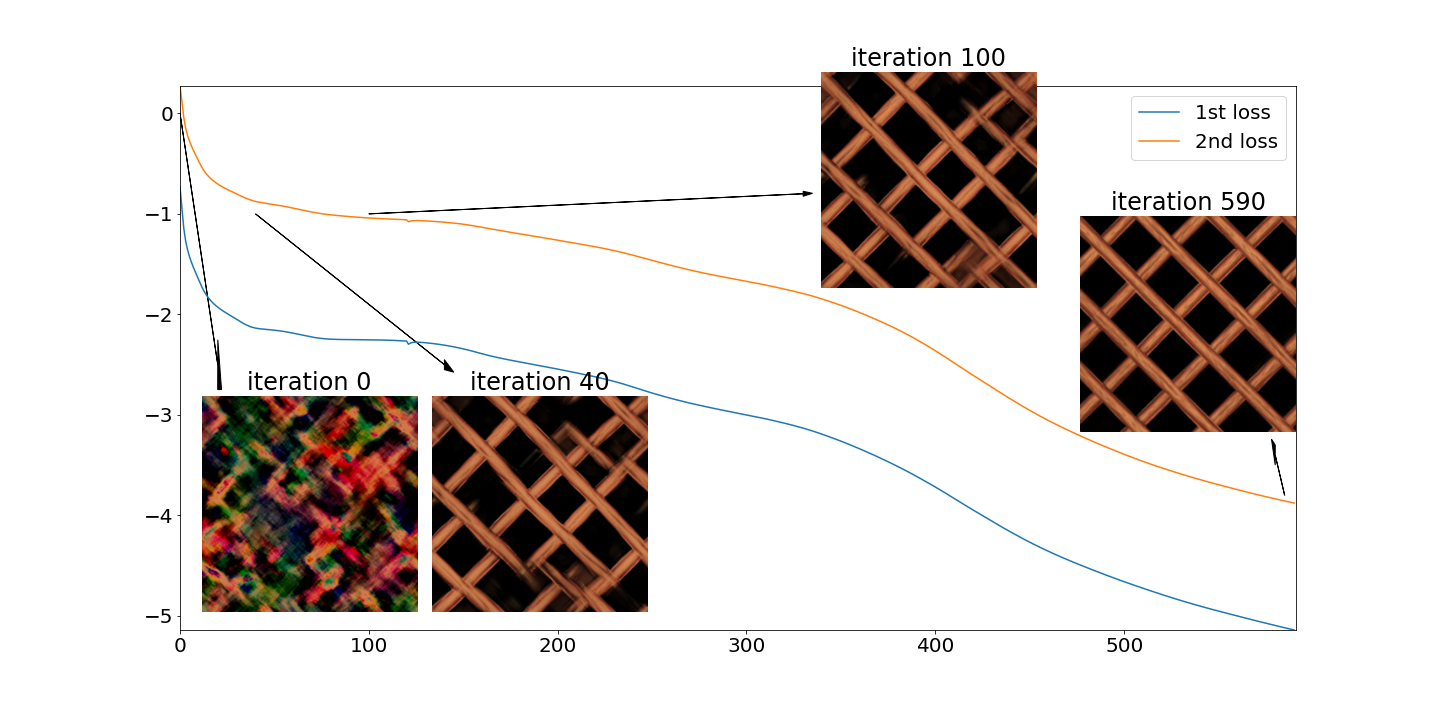}
    \caption{The synthesis process for different macro-textures with rigid patterns.}
    \label{fig: images 2}
\end{figure}

\begin{figure}
    \centering
    \includegraphics[width = 0.95 \linewidth]{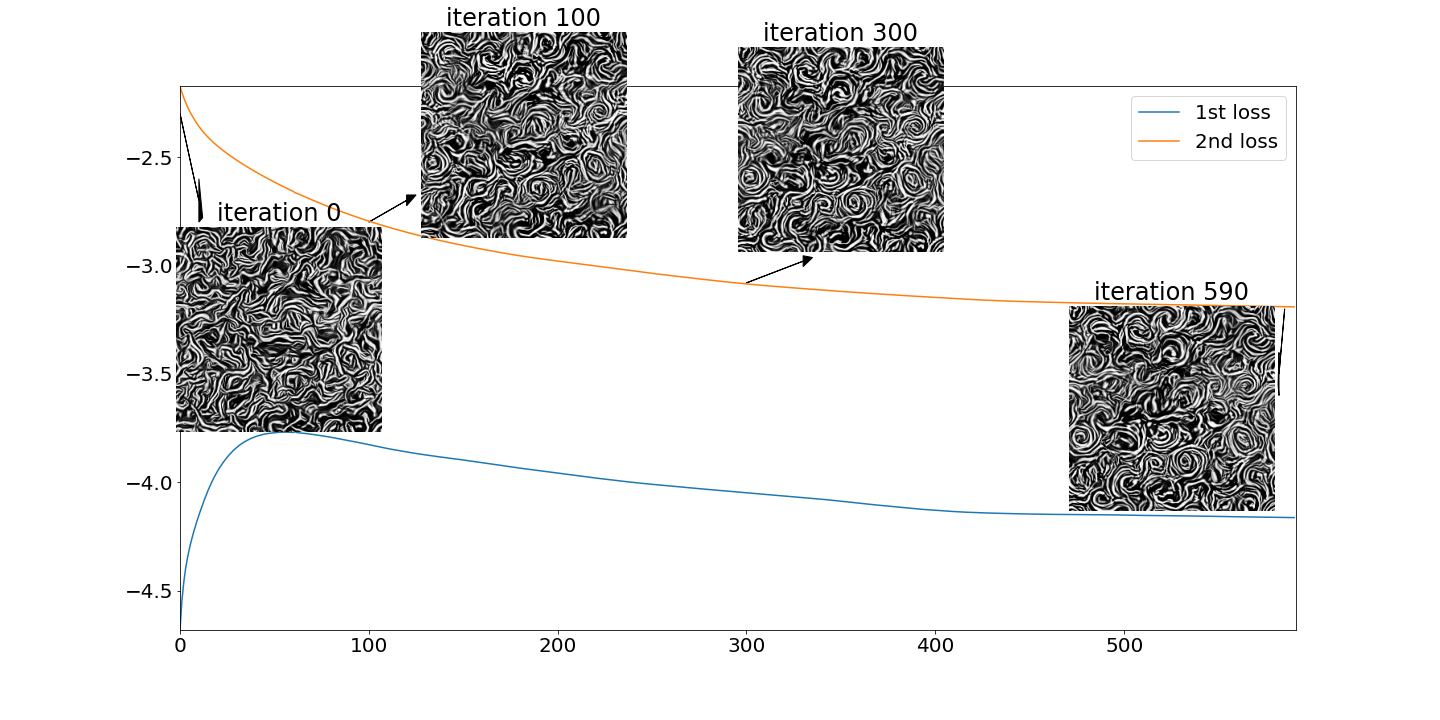}
    \caption{The synthesis process when initializing from the first layer synthesis, for a texture with complex patterns.}
    \label{fig: images 3}
\end{figure}

Figure \ref{fig: images 3} shows the synthesis process of a texture that has intricate patterns. Here we start from the first layer result to show how the second layer statistics improve the image quality. At iteration 0, which is essentially the synthesized image from the first layer experiment, there are barely swirls but only curves. As the algorithm goes on, we notice the first layer loss actually increases while the second layer loss is decreasing. The increasing first layer loss generally does not affect image quality so long as the second layer loss decreases. By the last iteration, the synthesized image contains longer and smoother swirls. 

\section{Additional numerical results}
\label{sec: additional numerical results}
We show more of our synthesis results in Figure \ref{fig:final additional}. These new images come from the same two databases in the main text: DTD database\footnote{https://www.robots.ox.ac.uk/~vgg/data/dtd/} and CG Texture database\footnote{https://www.textures.com/}. 

\begin{figure*}
    \centering
    \includegraphics[width = 0.45 \linewidth]{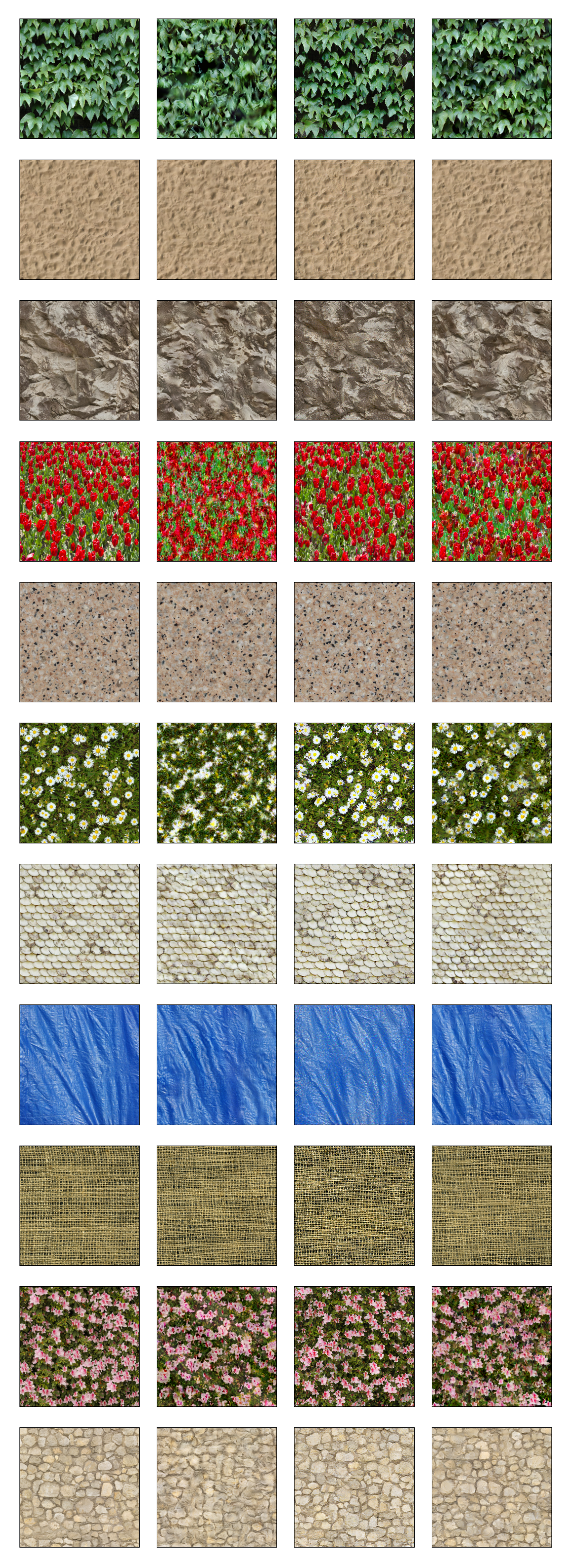}
    \includegraphics[width = 0.45 \linewidth]{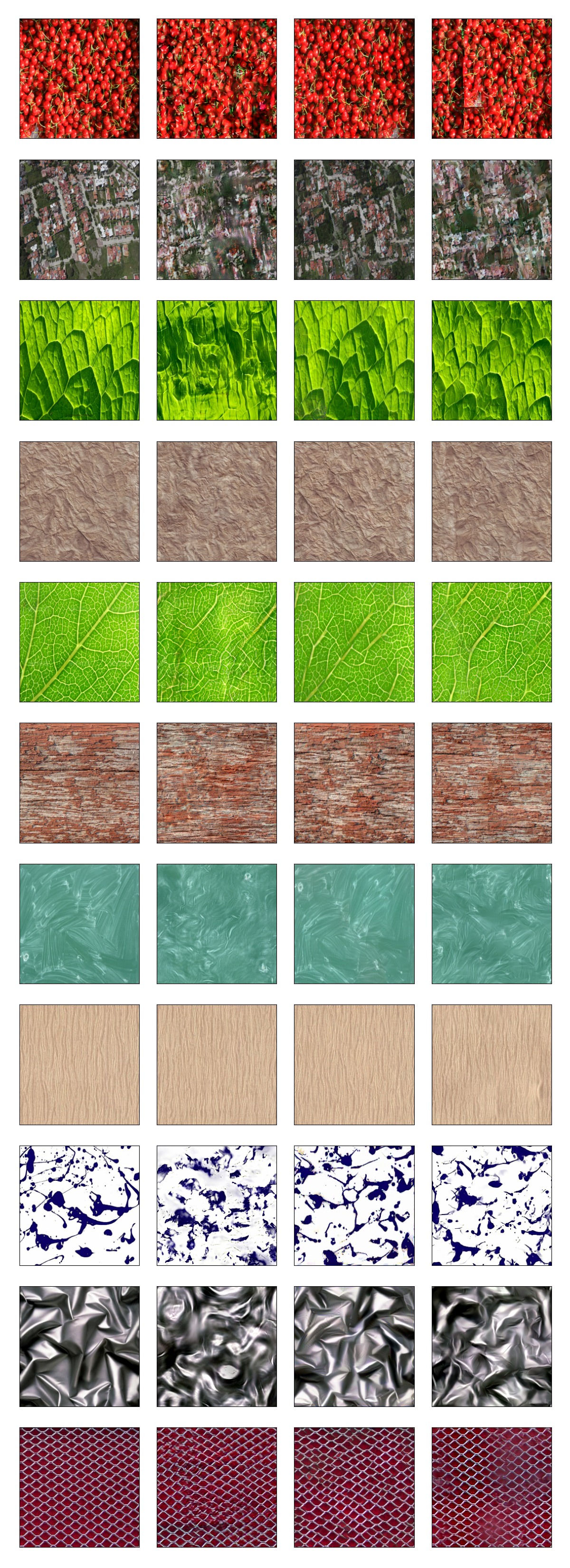}
    \caption{Synthesis results compared to other models. \textbf{Left:} Original images. \textbf{Middle Left:} Results from Portilla and Simoncelli \cite{Portilla}. \textbf{Middle Right: } Results from Gatys \textit{et al.} \cite{vggsyn}. \textbf{Right:} Results from our two layer model.}
    \label{fig:final additional}
\end{figure*}

\bibliographystyle{IEEEtran}
\bibliography{main_bib}

\clearpage
\end{document}